\documentclass[sigconf]{acmart}
\usepackage{cleveref}
\newtheorem{theorem}{Theorem}[section]
\theoremstyle{definition}
\newtheorem{definition}{Definition}[section]
\newcommand{\independent}{\perp\mkern-9.5mu\perp}
\newcommand{\notindependent}{\centernot{\independent}}
\newcommand{\RNum}[1]{\expandafter{\romannumeral #1\relax}}
\usepackage{enumerate}
\newtheorem{assumption}{A}
\usepackage{amsmath,empheq}
\newcommand{\matr}[1]{\mathbf{#1}}
\usepackage{algorithm}
\usepackage{algpseudocode}
\crefname{algocf}{alg.}{algs.}
\Crefname{algocf}{Algorithm}{Algorithms}
\usepackage{multirow}
\usepackage{subcaption}
\usepackage{graphicx}
\AtBeginDocument{%
  \providecommand\BibTeX{{%
    \normalfont B\kern-0.5em{\scshape i\kern-0.25em b}\kern-0.8em\TeX}}}

\setcopyright{none}
\copyrightyear{2018}
\acmYear{2018}
\acmDOI{XXXXXXX.XXXXXXX}

\acmConference[CIKM'24]{ACM CIKM 2024}{October 21 – 25,
  2024}{Boise, Idaho, USA}
\acmPrice{15.00}
\acmISBN{978-1-4503-XXXX-X/18/06}

\begin{document}

%%
%% The "title" command has an optional parameter,
%% allowing the author to define a "short title" to be used in page headers.
% \title{Causal Decomposition of State in Reinforcement Learning Recommender Systems}
\title{On Causally Disentangled State Representation Learning for Reinforcement Learning based Recommender Systems}
%%
%% The "author" command and its associated commands are used to define
%% the authors and their affiliations.
%% Of note is the shared affiliation of the first two authors, and the
%% "authornote" and "authornotemark" commands
%% used to denote shared contribution to the research.
\author{Siyu Wang}
\affiliation{%
  \institution{The University of New South Wales}
  \city{Sydney}
  \country{Australia}}
\email{siyu.wang5@unsw.edu.au}

\author{Xiaocong Chen}
\affiliation{%
  \institution{Data 61, CSIRO}
  \city{Eveleigh}
  \country{Australia}
}
\email{xiaocong.chen@data61.csiro.au}

\author{Lina Yao}
\affiliation{%
  \institution{Data 61, CSIRO}
  \city{Eveleigh}
  \country{Australia}}
\affiliation{
  \institution{The University of New South Wales}
  \city{Sydney}
  \country{Australia}}
\email{lina.yao@data61.csiro.au}

%%
%% By default, the full list of authors will be used in the page
%% headers. Often, this list is too long, and will overlap
%% other information printed in the page headers. This command allows
%% the author to define a more concise list
%% of authors' names for this purpose.
\renewcommand{\shortauthors}{Trovato and Tobin, et al.}

%%
%% The abstract is a short summary of the work to be presented in the
%% article.
\begin{abstract}
In Reinforcement Learning-based Recommender Systems (RLRS), the complexity and dynamism of user interactions often result in high-dimensional and noisy state spaces, making it challenging to discern which aspects of the state are truly influential in driving the decision-making process. This issue is exacerbated by the evolving nature of user preferences and behaviors, requiring the recommender system to adaptively focus on the most relevant information for decision-making while preserving generaliability.
To tackle this problem, we introduce an innovative causal approach for decomposing the state and extracting \textbf{C}ausal-\textbf{I}n\textbf{D}ispensable \textbf{S}tate Representations (CIDS) in RLRS. Our method concentrates on identifying the \textbf{D}irectly \textbf{A}ction-\textbf{I}nfluenced \textbf{S}tate Variables (DAIS) and \textbf{A}ction-\textbf{I}nfluence \textbf{A}ncestors (AIA), which are essential for making effective recommendations. By leveraging conditional mutual information, we develop a framework that not only discerns the causal relationships within the generative process but also isolates critical state variables from the typically dense and high-dimensional state representations. We provide theoretical evidence for the identifiability of these variables. Then, by making use of the identified causal relationship, we construct causal-indispensable state representations, enabling the training of policies over a more advantageous subset of the agent's state space.
We demonstrate the efficacy of our approach through extensive experiments, showcasing our method outperforms state-of-the-art methods.
\end{abstract}

%%
%% The code below is generated by the tool at http://dl.acm.org/ccs.cfm.
%% Please copy and paste the code instead of the example below.
%%

%%
%% Keywords. The author(s) should pick words that accurately describe
%% the work being presented. Separate the keywords with commas.

% \received{20 February 2007}
% \received[revised]{12 March 2009}
% \received[accepted]{5 June 2009}

%%
%% This command processes the author and affiliation and title
%% information and builds the first part of the formatted document.
\maketitle

\section{Introduction}

Recommender Systems (RS) are crucial in navigating the vast digital environment, tailoring suggestions to align with individual user preferences. The integration of Reinforcement Learning (RL) within RS, known as RLRS, has transformed the recommendation experience into a dynamic, sequential decision-making process. Unlike traditional systems that passively suggest content based on static data, RLRS actively engage with users, continuously adapting recommendations in response to real-time feedback. This approach seeks to maximize long-term user engagement by treating each interaction as an opportunity to learn and enhance the personalization of content, thereby maintaining alignment with the evolving preferences and behaviors of the users.

In RLRS, the efficiency hinges on three essential components: State Representation, Policy Optimization, and Reward Formulation~\cite{CHEN2023110335, afsar2022reinforcement}. While much of the current research in RLRS is centered on policy optimization~\cite{zhang2017dynamic, wang2020hybrid, chen2020knowledge,chen2024maximum} and reward formulation~\cite{chen2021generative, DBLP:conf/iclr/KostrikovADLT19, chen2018stabilizing}, the role of state representation should not be understated. The state in RLRS is a composite of varied attributes: user characteristics (like age, gender, and recent activities), item properties (including price, category, and popularity), and contextual elements (such as time and location). Effectively distilling this rich tapestry of information poses a significant challenge. Neglecting key features might result in suboptimal recommendations, while incorporating too much detail could clutter the system with irrelevant data, diminishing its predictive accuracy.
Recent advances in representation learning algorithms in RL aim to extract abstract features from high-dimensional data, which has proven beneficial in enhancing the efficiency of existing RL algorithms~\cite{lesort2018state, gelada2019deepmdp, huang2022action}. For instance,~\citet{zhang2021learning} developed a method to learn representations by ignoring task-irrelevant information through the bisimulation metric, while~\citet{pmlr-v162-wang22ae} proposed a state abstraction technique in model-based RL by learning a dynamic model that minimizes dependencies between state variables and actions. Despite these advancements in RL, the exploration of state representation in RLRS is still limited. RLRS often involves complex, high-dimensional data and intricate causal relationships within the recommender system. Rather than merely condensing the state aggregation, we aim to discern and highlight the specific state dimensions that are causally critical for decision-making in RLRS, providing a more targeted and effective approach to recommendation processes.

In RLRS, the agent's actions involve recommending items, while the rewards typically correspond to user feedback like clicks, purchases, or exits. However, rewards alone do not clearly indicate which aspects of the state influence user behavior, making it challenging to discern critical from irrelevant state dimensions for effective decision-making. To tackle this, we introduce Causal-Indispensable State Representations (CIDS) in RLRS. CIDS leverages causal relationships between actions and state variables, as well as among different state dimensions, to identify key elements crucial for policy learning.

CIDS focuses on two types of causal relationships. The first involves state dimensions directly influenced by actions, termed Directly Action-Influenced State Variables (DAIS). For instance, when a user's recent browsing history changes following specific item recommendations, it indicates a direct impact of the action on this state dimension. Such changes can extend to item features like popularity or category trends, offering valuable feedback on user preferences in response to recommendations.
The second type of causal relationship in CIDS pertains to state dimensions that affect DAIS. Since DAIS are pivotal for policy learning, other dimensions influencing DAIS are also deemed critical. This relationship encompasses interactions between various state dimensions, such as how a user's browsing history might correlate with their likelihood of engaging with similar items, or how static attributes like age and gender could inform preferences in certain categories.
This paper employs causal graphical models to uncover and identify these causal relationships within CIDS. We theoretically validate that CIDS can be determined using conditional dependence and independence relationships and can be learned from observations through conditional mutual information. Our main contributions are:
\begin{itemize}
    \item Reformulating the RLRS process using a causal graphical model to describe the interactions between actions and state dimensions.
    \item Theoretically characterizing CIDS by leveraging conditional dependence and independence relationships.
    \item Proposing a methodology to learn the underlying causal structure of CIDS from observations using conditional mutual information, thereby constructing CIDS for more efficient and effective recommendation policy learning.
    \item Demonstrating our method's efficacy on both online simulators and offline datasets, showing improvements in the efficiency and performance of recommendation policy learning through the application of CIDS.
\end{itemize}

% \end{itemize}

%----------------------------------------------------------------------------------
\section{Preliminaries}
\label{sec:pre}
%-------------------------------------------------------------------------

\subsection{Reinforcement Learning based Recommender Systems}
Reinforcement Learning (RL) in Recommender Systems (RS) is focused on optimizing decision-making through ongoing user interaction, framed within a Markov Decision Process (MDP) modelled by the tuple \( \langle \mathcal{S}, \mathcal{A}, \mathcal{R}, \mathcal{P}, \gamma \rangle \). Here, \( \mathcal{S} \) denotes the state space, containing user data, historical interactions, item characteristics, and contextual elements;
\( \mathcal{A} \) represents the action space, including all candidate items; \( \mathcal{R}: \mathcal{S} \times \mathcal{A} \to \mathbb{R} \) is the reward function, related to the user feedback; \( \mathcal{P} \) covers the transition probabilities of transitioning from one state to another; and \( \gamma \) is the discount factor, prioritizing immediate versus future rewards.

In the MDP framework, the agent (RS) interacts with its environment over discrete time steps denoted by \( t = 0, 1, 2, \ldots, n \). At each time step \( t \), the agent examines the current state \( s_t \), which encompasses user preferences, past interactions, and item information, all contained within the state space \( \mathcal{S} \). The agent then selects an action \( a_t \) from the action space \( \mathcal{A}(s_t) \), often comprising a set of item recommendations. This action prompts a transition to a subsequent state \( s_{t+1} \), and the agent receives a corresponding reward \( r_t \), reflecting the user’s response and the effectiveness of the recommended items. The RS strives to establish a policy \( \pi: \mathcal{S} \to \mathcal{A} \) that optimizes the cumulative discounted return, thus assessing the long-term effectiveness of its recommendations.

%-------------------------------------------------------------------------
\subsection{Causal Graphical Model}
Let's formally define the causal graphical model followed by~\cite{10.5555/3202377}.
Consider a set of finitely many random variables denoted as $\mathbf{X} = (X^1, ..., X^d)$ with an index set $\mathbf{V}:=\{1, ..., d\}$. These random variables have a joint distribution $P_\mathbf{X}$ and a density function $p(\mathbf{x})$.

A causal graphical model is represented by a Directed Acyclic Graph
(DAG) $\mathcal{G} = (\mathbf{V}, \mathcal{E})$, where $\mathbf{V}$ represents the nodes or vertices of the graph, and $\mathcal{E}$ represents the edges between the nodes. The edges $\mathcal{E} \subseteq \mathbf{V}^2$ satisfy the property that for any node $v \in \mathbf{V}$, $(v, v) \notin \mathcal{E}$, meaning there are no self-loops in the graph.
In a causal graph, a random variable $X^i$ is considered a direct cause of $X^j$ if and only if $(i, j) \in \mathcal{E}$ and $(j, i) \notin \mathcal{E}$. 
% This implies that if there is a directed edge from node i to node j, but no directed edge from node j to node i, then Xᵢ is considered a direct cause of Xⱼ.
Hence, it is assumed that the causal graph is acyclic, meaning that there are no directed cycles in the graph. This ensures that there are no causal loops where the causal influence could propagate indefinitely.

In a DAG, two disjoint subsets of vertices, denoted as $\mathbf{A}$ and $\mathbf{B}$, are considered d-separated (see~\Cref{def-d} in appendix) by a third disjoint subset $\mathbf{S}$ if all paths connecting nodes in $\mathbf{A}$ and $\mathbf{B}$ are blocked by $\mathbf{S}$. This relationship is denoted as $\mathbf{A} \independent_G \mathbf{B} | \mathbf{S}$.

%----------------------------------------------------------------------------------
\section{Methodology}
%--------------------------------------------------------------------
\subsection{Causal-indispensable State Representation}
\label{SEC-CISR}

\begin{figure}
    \centering
    \includegraphics[width=0.9\linewidth]{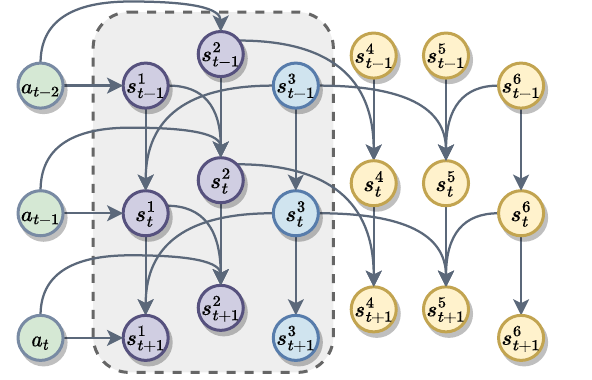}
    \caption{An illustrative causal graphical model for an MDP is depicted. The state $s_t$ is decomposed into six different dimensions, denoted as $s_t = (s^1_t, \ldots, s^6_t)$. The purple nodes signify DAIS, while the blue nodes symbolize AIA. The nodes enclosed within the gray area collectively represent CIDS, which contains both DAIS and AIA.}

    \label{fig:causal}
\end{figure}
To characterize a set of causal-indispensable state representations for RLRS, we consider decomposing $s_t$ into $d$ different dimensions denoted as $s_t = (s^1_t, ..., s^d_t)$.
Given the causal graphical model for one-step MDP $\mathcal{G} = (\mathbf{V}, \mathcal{E})$, where $\mathbf{V} = \{a_t, s^{1:d}_t, s^{1:d}_{t+1}\}$ represents the nodes of the graph, and $\mathcal{E}$ represents the edges describing the
causal relationships between actions and different dimensions of states.
It is commonplace that the action variable may not influence every dimension of the state variable, and there are structural relationships among different dimensions of $s_t$. 
For example, consider the causal graphical model shown in~\Cref{fig:causal}, $a_{t-1}$ is a direct cause of $s^1_t$ and $s^2_t$, and $s^2_t$ has an impact on $s^4_{t+1}$, while there is no edge from $a_{t-1}$ to $s^5_{t}$.
To reflect the different structural relationships between actions and different dimensions of states, and among different dimensions of states, we consider the decomposition of the state $s_t$ as follows:

\begin{definition}[Causal Decomposition of State]
\label{CIDS}
    Given the causal graphical model, such as the model in~\Cref{fig:causal}, that describes the causal relationship among the actions and dimensions of states within the MDP environment, the state can be decomposed into two subsets: \textbf{C}ausal-\textbf{I}n\textbf{D}ispensable \textbf{S}tate Representations (CIDS) and causal-dispensable state representations, where the CIDS is defined as including the subsets of state dimensions satisfying one of the following causal relationships:
    \begin{enumerate}[(i)]
    \item \textbf{D}irectly \textbf{A}ction-\textbf{I}nfluenced \textbf{S}tate Variables (DAIS): A state variable \( s^i_t \in \mathcal{S} \) belongs to $\text{DAIS}_t$ if there exists a direct edge from \( a_{t-1} \in \mathcal{A} \) to \(  s^i_t \) in the causal graph. 
    \item \textbf{A}ction-\textbf{I}nfluence \textbf{A}ncestors (AIA): A state variable \( s^j_t \in \mathcal{S} \) belongs to $\text{AIA}_t$ if it is an ancestor of any state variable in DAIS in the causal graph. That is, there exists a directed path in the causal graph from \( s^j_t \) to some \( s^i_{t+1} \in \text{DAIS}_{t+1} \), and there is no direct edge from any action variable \( a_{t-1} \in \mathcal{A} \) to \( s^j_t \).
    \end{enumerate}
    Correspondingly, a state variable \( s^m_t \in \mathcal{S} \) belongs to causal-dispensable state representations (CDS) if it is neither a part of DAIS nor AIA. That is, \( s^m_t \) has no direct causal relationship with action variable \( a_t \in \mathcal{A} \) or any state in DAIS.
\end{definition}

In RLRS, the state often includes a variety of user and item features. RL agents learn to choose appropriate actions according to the current state vector $s_t$ to improve user satisfaction and engagement, in which some dimensions may be redundant for policy learning. In an RLRS, the policy dictates how recommendations are made based on the current state. Utilizing CIDS would mean that the policy learning is based on the most causally indispensable and sufficient state representations when the structural relationship is given. \Cref{fig:causal} shows different types of state dimensions in the example causal structural relationship. The state variables in purple nodes belong to DAIS, as these are state variables that have a direct causal relationship with action variables (e.g. the historical interaction or the item popularity will be impacted after the action).

In other words, they are the descendants of action variables in the causal graph. These state dimensions change directly in response to the actions taken by the agent. The state variables in blue nodes belong to AIA, as these are state variables that are ancestors of the DAIS in the causal graph. They do not directly interact with action variables but have an impact on the DAIS (e.g. user age or gender can influence preferences and, consequently, historical interaction). They form a preceding layer in the causal structure, influencing the states that are directly affected by the agent's actions. 

Then we demonstrate how the proposed CIDS can be determined by using the conditional dependence and independence relationship among the variables under the following assumptions~\cite{mastakouri2021necessary}:

\begin{assumption}
    Markov condition\footnote{Here we use the global version of Markov condition in \Cref{Markov}(i)} and faithfulness
     for the underlying DAG.
\end{assumption}

\begin{assumption}
    There is an arrow from $s^i_{t-1}$ to $s^i_{t}$.
\end{assumption}

\begin{assumption}
    There are no long-range arrows, i.e. arrows from $s^i_{t-m}$ to $s^i_{t}$ for any $m>1$, and no backward arrows in time.
\end{assumption}

\begin{assumption}
    There are no arrows between state variables at the same timestep, i.e., no arrow from \( s^i_{t} \) to \( s^j_{t} \) for all \( i \) and \( j \).
\end{assumption}

The assumption A1 allows us to have a one-to-one correspondence between d-separation statements in the graph $\mathcal{G}$ and the corresponding conditional independence statements in the distribution. The assumptions A2-A4 impose some restrictions on the connectivity of the graph with respect to the properties of MDP.

\begin{theorem}
    Under the assumptions A1-A4, \( s^i_{t+1} \in \text{DAIS}_{t+1} \) if and only if \( a_{t} \notindependent s^i_{t+1} | \text{DAIS}_{t} \).
    \label{th-DAIS}
\end{theorem}

\begin{proof}(Proof by Contradiction)

Forward Direction (\( \Rightarrow \)):

Suppose \( s^i_{t+1} \in \text{DAIS}_{t+1} \). By definition, there exists a direct edge from \( a_{t} \) to \( s^i_{t+1} \) in the causal graph. To prove by contradiction, assume that \( a_{t} \independent s^i_{t+1} | \text{DAIS}_{t} \).

However, \( a_{t} \independent s^i_{t+1} | \text{DAIS}_{t} \) would violate the assumption A1, under which the direct edge from \( a_{t} \) to \( s^i_{t+1} \) implies that \( a_{t} \) and \( s^i_{t+1} \) are not conditionally independent given any set that doesn't include one of them.

Therefore, we show that if \( s^i_{t+1} \in \text{DAIS}_{t+1} \), then \( a_{t} \notindependent s^i_{t+1} | \text{DAIS}_{t} \).

Backward Direction (\( \Leftarrow \)):

Suppose \( a_{t} \notindependent s^i_{t+1} | \text{DAIS}_{t} \), implying a direct or indirect influence from \( a_{t} \) to \( s^i_{t+1} \).
For contradiction, assume \( s^i_{t+1} \notin \text{DAIS}_{t+1} \), meaning there is no direct edge from \( a_{t} \) to \( s^i_{t+1} \).

Considering assumptions A2-A4, the only permissible connection from \( a_{t} \) to \( s^i_{t+1} \) under these constraints is a direct edge, since no long-range arrows, backward arrows, or arrows between state variables at the same timestep are permitted.
Consequently, if \( a_{t} \notindependent s^i_{t+1} | \text{DAIS}_{t} \), it implies the existence of a direct edge from \( a_{t} \) to \( s^i_{t+1} \), necessitating that \( s^i_{t+1} \) is indeed a part of \( \text{DAIS}_{t+1} \). This stands in contradiction to our initial assumption that \( s^i_{t+1} \notin \text{DAIS}_{t+1} \).

Therefore, it is established that if \( a_{t} \notindependent s^i_{t+1} | \text{DAIS}_{t} \), then it must follow that \( s^i_{t+1} \in \text{DAIS}_{t+1} \).

By contradiction in both directions, the theorem is proven. Under assumptions A1-A4, \( s^i_{t+1} \in \text{DAIS}_{t+1} \) if and only if \( a_{t} \notindependent s^i_{t+1} | \text{DAIS}_{t} \).
\end{proof}
\begin{theorem}
        Under the assumptions A1-A4, \( s^i_{t-1} \in \text{AIA}_{t-1} \) if and only if $a_{t-1} \notindependent s^i_{t-1} | \text{DAIS}_t$
        \label{th-AIA}
\end{theorem}

\begin{proof}(Prove by Contradiction)

Forward Direction (\( \Rightarrow \)):

    Suppose \( s^i_{t-1} \in \text{AIA}_{t-1} \), meaning there's a directed edge from \( s^i_{t-1} \) to some \( s^k_{t} \in \text{DAIS}_t \). Given \( s^k_{t} \in \text{DAIS}_t \), there is a directed edge from \( a_{t-1} \) to \( s^k_{t} \). Hence, there is a path between \( a_{t-1} \) and \( s^i_{t-1} \in \text{AIA}_{t-1} \): \( a_{t-1} \rightarrow s^k_{t} \leftarrow s^i_{t-1} \).
    
    Assume for contradiction that \( a_{t-1} \independent s^i_{t-1} | \text{DAIS}_t \), which suggests that \( \text{DAIS}_t \) blocks every path between \( a_{t-1} \) and \( s^i_{t-1} \). However, this contradicts the path \( a_{t-1} \rightarrow s^k_{t} \leftarrow s^i_{t-1} \), which is not be blocked by \( \text{DAIS}_t \) (see~\Cref{def-d}).
    
    Therefore, we conclude that if  \( s^i_{t-1} \in \text{AIA}_{t-1} \), then \( a_{t-1} \notindependent s^i_{t-1} | \text{DAIS}_t \).

Backward Direction (\( \Leftarrow \)):

    Suppose \( a_{t-1} \notindependent s^i_{t-1} | \text{DAIS}_t \) and assume for contradiction that \( s^i_{t-1} \notin \text{AIA}_{t-1} \).

    If \( s^i_{t-1} \notin \text{AIA}_{t-1} \), it implies there is no directed path from \( s^i_{t-1} \) to any state in \( \text{DAIS}_t \).
    However, \( a_{t-1} \notindependent s^i_{t-1} | \text{DAIS}_t \) indicates at least one path between \( a_{t-1} \) and \( s^i_{t-1} \) that is not blocked by \( \text{DAIS}_t \). Under assumptions A2-A4, we can see that the path between \( a_{t-1} \) and \( s^i_{t-1} \)  suggests existing a directed path from \( s^i_{t-1} \) to some state in \( \text{DAIS}_t \).
    This directed path contradicts the assumption that \( s^i_{t-1} \notin \text{AIA}_{t-1} \).
   Therefore, if \( a_{t-1} \notindependent s^i_{t-1} | \text{DAIS}_t \), then \( s^i_{t-1} \) must be in \( \text{AIA}_{t-1} \).

By proving by contradiction in each direction, we establish the theorem. Under the assumptions A1-A4, \( s^i_{t-1} \in \text{AIA}_{t-1} \) if and only if $a_{t-1} \notindependent s^i_{t-1} | \text{DAIS}_t$
\end{proof}

%-------------------------------------------------------------------------
\subsection{Causal-indispensable State Representation Learning}
\label{CIDS_Learning}
%-------------------------------------------------------------------------

When learning the CIDS in practice, a significant challenge arises from the unknown nature of the causal graphical model. Therefore, instead of learning CIDS directly from the causal graphical model, we rely on the collected transition data \(\mathcal{D} = \{(s^l_t, a^l_t, s^l_{t+1}, r^l_t)\}\). 
To tackle this, we initially define the causal structure of a DAG $\mathcal{G}$, which describes the causal relationships within the MDP environment. This structure explicitly encodes the causal relationships between actions and various state dimensions, alongside the relationships among state dimensions themselves:

\begin{equation}
    s^j_{t+1} = f(\matr{M}^{(\cdot, j)}_{s \to s} \odot s_t, \matr{M}^{(,j)}_{a \to s} \odot a_t, u^{s,j}_t), \quad \text{for } j = 1, \ldots, d.
    \label{transition}
\end{equation}

Here, \(s^j_{t+1}\) is a dimension of the state at time \(t+1\), and \( s_{t+1} = (s^1_{t+1}, \ldots, s^d_{t+1})\). The binary mask \(\matr{M}\) captures the causal structure, and $\odot$ denotes the element-wise multiplication. The matrix \(\matr{M}_{s \to s} \in \{0, 1\}^{d \times d}\) indicates causal relationships between dimensions of the state \(s_t\) and the next state \(s_{t+1}\). Specifically, \(\matr{M}^{(i, j)}_{s \to s} = 1\) implies an edge from \(s^i_t\) to \(s^j_{t+1}\). Similarly, the matrix \(\matr{M}_{a \to s} \in \{0, 1\}^{1 \times d}\) represents causal relationships between the action \(a_t\) and the dimensions of the next state \(s_{t+1}\). The notation \(\matr{M}^{(,j)}_{a \to s}\) denotes the \(j\)-th column of this matrix. We operate under the assumption that the causal structure linking \(s_t\), \(s_{t+1}\), and \(a_t\) is time-invariant without any unobserved confounders.

Based on \Cref{CIDS}, DAIS comprises a set of state variables that have direct edges from the action \( a_{t-1} \). We can represent DAIS using the causal structure mask \(\matr{M}\) as:
$\text{DAIS}_t = \matr{M}_{a \to s} \odot s_t$.
Similarly, AIA consists of state variables at time \( t \) that have a causal influence on any state variable in \( \text{DAIS}_{t+1} \) at the next time step. Accordingly, AIA can be formulated as $\text{AIA}_t = \matr{M}_{s \to s} \odot s_t$.

As discussed in \Cref{SEC-CISR}, \(\text{DAIS}_{t}\) and \(a_{t-1}\) are not conditionally independent given \(\text{DAIS}_{t-1}\). In contrast, other state dimensions, except for \(\text{DAIS}_{t}\), are conditionally independent given \(\text{DAIS}_{t-1}\). Similarly, for the Action-Influence Ancestors (AIA), \(\text{AIA}_{t-1}\) and \(a_{t-1}\) are not conditionally independent given \(\text{DAIS}_{t}\). However, other state dimensions, excluding \(\text{AIA}_{t}\), exhibit conditional independence given \(\text{DAIS}_{t}\). 
We proceed to formalize the learning process CIDS using the principles of Conditional Mutual Information (CMI). We aim to maximize the CMI to learn the DAIS and AIA in each case.

The learning process for DAIS can be formulated by maximizing the following objective:
\begin{equation}
    \mathcal{L}_{DAIS} = I(\text{DAIS}_{t+1} ; a_{t} | \text{DAIS}_{t}),
\end{equation}
where \(I(\cdot)\) denotes the Mutual Information. This equation captures the mutual information between \(\text{DAIS}_{t+1}\) and \(a_{t}\) conditioned on \(\text{DAIS}_{t}\).

Subsequently, AIA can be learned by maximizing the following:
\begin{equation}
    \mathcal{L}_{AIA} = I(\text{AIA}_{t-1} ; a_{t-1} | \text{DAIS}_{t}),
\end{equation}
which similarly utilizes mutual information measures to identify the most relevant state dimensions constituting \(\text{AIA}_{t-1}\).

The mutual information \( I(\text{DAIS}_{t+1}; a_{t} | \text{DAIS}_{t}) \) can be expressed in terms of conditional entropy, which is defined as the difference between the conditional entropy of one variable given the conditioning variable and the conditional entropy of the same variable given both the conditioning variable and the second variable. 
Thus, it can be defined using conditional entropy as follows: 

\begin{equation}
\begin{aligned}
    I(\text{DAIS}_{t+1} ; a_{t} | \text{DAIS}_{t}) &= H(\text{DAIS}_{t+1} | \text{DAIS}_{t}) \\
    &- H(\text{DAIS}_{t+1} | a_{t}, \text{DAIS}_{t}).
\end{aligned}
\end{equation}

By leveraging the probabilistic predictive model of DAIS parameterized by \(\theta_{D_i}\).
Then the conditional entropy of \(\text{DAIS}_{t+1}\) given \(\text{DAIS}_{t}\) is calculated as:

\begin{equation}
\begin{aligned}
& H(\text{DAIS}_{t+1} | \text{DAIS}{t})\\
& = -\mathbb{E}_{s_t, a_t, s_{t+1}\sim
    \mathcal{D}} \left[ \log P(\text{DAIS}_{t+1} | \text{DAIS}_{t}; \theta_{D_1}) \right]\\
&= -\mathbb{E}_{s_t, a_t, s_{t+1}\sim
    \mathcal{D}} \left[ \sum_{j=1}^d \log P(s^j_{t+1} \in \text{DAIS}_{t+1} | s_t, \matr{M}^{(,j)}_{a \to s} ; \theta_{D_1}) \right]
\end{aligned}
\end{equation}

The conditional entropy of \(\text{DAIS}_{t+1}\) given both \(a_{t}\) and \(\text{DAIS}_{t}\) is calculated similarly, but conditioning on both the previous action and \(\text{DAIS}_{t}\):

\begin{equation}
\begin{aligned}
& H(\text{DAIS}_{t+1} | a_{t}, \text{DAIS}{t})\\
& = -\mathbb{E}_{s_t, a_t, s_{t+1}\sim
    \mathcal{D}} \left[ \log P(\text{DAIS}_{t+1} | a_{t}, \text{DAIS}_{t} ; \theta_{D_2}) \right]\\
&= -\mathbb{E}_{s_t, a_t, s_{t+1}\sim
    \mathcal{D}} \left[ \sum_{j=1}^d \log P(s^j_{t+1} \in \text{DAIS}_{t+1} | a_{t}, s_t, \matr{M}^{(,j)}_{a \to s} ; \theta_{D_2}) \right]
\end{aligned}
\end{equation}

 For the AIA and the corresponding mutual information term, the expression can be reformulated in a similar way:

\begin{equation}
\begin{aligned}
&I(\text{AIA}_{t-1} ; a_{t-1} | \text{DAIS}_{t}) \\
= & H(\text{AIA}_{t-1} | \text{DAIS}_{t}) - H(\text{AIA}_{t-1} | a_{t-1}, \text{DAIS}_{t})\\
= & -\mathbb{E}_{s_{t-1}, a_{t-1}, s_{t} \sim \mathcal{D}} \left[ \sum_{j=1}^d \log P(s^j_{t-1} \in \text{AIA}_{t-1} | s_t, \matr{M}^{(,j)}_{a \to s}; \theta_{A_1}) \right] \\
& + \mathbb{E}_{s_{t-1}, a_{t-1}, s_{t} \sim \mathcal{D}} \left[ \sum_{j=1}^d \log P(s^j_{t-1} \in \text{AIA}_{t-1}  | a_{t-1}, s_t, \matr{M}^{(,j)}_{a \to s}; \theta_{A_2}) \right],
\end{aligned}
\end{equation}

where \(\theta_{A_i}\) are the parameters of the probabilistic predictive model of AIA. 

\begin{theorem}(Identifiability of Causal Structure)
Given the observable state \( s_t \) and action \( a_t \), which form the equation in~\Cref{transition}, the causal structure masks \( \matr{M}_{s \to s} \) and \( \matr{M}_{a \to s} \)
% and the transition dynamics \( f \) 
are identifiable under the global Markov condition and faithfulness assumption.
~\label{identifiability}
\end{theorem}

The proof of~\Cref{identifiability} is detailed in~\Cref{app:th3.3}. This theorem lays the theoretical groundwork necessary for identifying the causal structure masks from observed data.

\subsection{Objective Function}
\label{sec:obj}
The predictive models are designed with distinct objectives for learning the DAIS and AIA components, parameterized by $\theta_{D}$ and $\theta_{A}$ respectively. The objective function for each component is tailored to minimize loss while incorporating a regularization term to promote sparsity in the learned structural matrices.

The objective function for the DAIS predictive model is given by:
\begin{equation}
    \mathcal{L}_{\text{DAIS-model}} = -\mathcal{L}_{\text{DAIS}} + \lambda_1 \| \matr{M}_{a \to s} \|_1,
    \label{eq:l_dais}
\end{equation}
where $\mathcal{L}_{\text{DAIS}}$ is 
%the loss 
related to the learning of DAIS and $\lambda_1$ is a hyperparameter that controls the sparsity of the action-to-state causal structure matrix.

For the AIA predictive model, the objective function is:
\begin{equation}
    \mathcal{L}_{\text{AIA-model}} = -\mathcal{L}_{\text{AIA}} + \lambda_2 \| \matr{M}_{s \to s} \|_1,
    \label{eq:l_aia}
\end{equation}
where $\mathcal{L}_{\text{AIA}}$ is 
%the loss 
associated with the learning of AIA and $\lambda_2$ is a hyperparameter for the sparsity of the state-to-state causal structure matrix.

The detailed formulations for $\mathcal{L}_{\text{DAIS}}$ and $\mathcal{L}_{\text{AIA}}$ are provided in~\Cref{CIDS_Learning}, while the comprehensive expressions for both objective functions are available in~\Cref{app:obj}.

\subsection{Recommendation Policy Learning}
Upon establishing the DAIS and AIA established, we can construct the causal-indispensable state representation by combining these elements: \( \text{CIDS} = (\text{DAIS}, \text{AIA}) \). To operationalize this, we define a causal structure matrix for CIDS such that \( \matr{M}_{\text{CIDS}} = \{\matr{M}_{s \to s}\} \vee \{\matr{M}_{a \to s}\} \). The CIDS at time \( t \) can then be succinctly expressed as \( \text{CIDS}_t = \matr{M}_{\text{CIDS}} \odot s_t \), which serves as the input for learning the recommendation policy. In this way, the recommendation policy chooses the action only with the causal-indispensable state dimensions.
The detailed steps of this comprehensive learning approach are outlined in~\Cref{alg:al1}, encompassing three critical phases: (\RNum{1}) data collection using a sub-optimal policy, (\RNum{2}) learning of DAIS and AIA from the collected data, and (\RNum{3}) learning of recommendation policy with CIDS.

\begin{algorithm}
\caption{Training Procedure for Predictive Models and Policy Learning}
\begin{algorithmic}[1]
\State Input: Dataset $\mathcal{D}$, empty reply buffer $\mathcal{B}$, initial neural networks parameters $\theta_{D}$ and $\theta_{A}$, initial recommendation policy parameter $\theta_{\text{rec}}$, hyperparameters $\lambda_1, \lambda_2$

\For{episode = 1, ..., E}
    \For{t = 1, ..., T}
        \State Sample a random minibatch of $K$ trajectories from $D$
        \State // Optimize DAIS Predictive Model
        \State Update $\theta_{D}$ by minimizing $\mathcal{L}_{\text{DAIS-model}}$(~\cref{eq:l_dais}) with minibatch
        \State // Optimize AIA Predictive Model
        \State Update $\theta_{A}$ by minimizing $\mathcal{L}_{\text{AIA-model}}$(~\cref{eq:l_aia}) with minibatch
    \EndFor
\EndFor
\State // Training of the recommendation policy $\pi$
\For{episode = 1, ..., E}
    \State  Receive initial observation state $s_1$\;
    \For{t = 1, ..., T}
        \State Calculate $\matr{M}_{\text{CIDS}}$ based on $\matr{M}_{s \to s}$ and $\matr{M}_{a \to s}$
        \State Observe state $s_t$ and select action $a \sim \pi(\matr{M}_{\text{CIDS}} \odot s_t)$
        \State Execute $a_t$ in the environment
        \State Observe next state $s_{t+1}$ and receive reward $r_t$
        \State Store transition $(\matr{M}_{\text{CIDS}} \odot s_t, a_t, s_{t+1}, r_t)$ in $\mathcal{B}$
        \State Sample a random minibatch of $K$ transition from $D$
        \State Update recommendation policy parameter $\theta_{\text{rec}}$
    \EndFor
\EndFor
\label{alg:al1}
\end{algorithmic}
\end{algorithm}

\section{Experiments}
In this section, we begin by performing experiments on an online simulator and recommendation datasets to highlight the remarkable performance of our methods. We then conduct an ablation study to demonstrate the effectiveness of the causal-indispensable state representation. 

\subsection{Experimental Setup}
We introduce the experimental settings with regard to environments and state-of-the-art RL methods. The implementation details can be found in~\Cref{app:imp}.

\subsubsection{Recommendation Environments}
\paragraph{Online Evaluation}
For online evaluation, we employ VirtualTaobao~\cite{shi2019virtual}, a simulation platform that replicates an online retail environment. This platform leverages data from Taobao, one of China's largest online retail sites, using hundreds of millions of genuine data points. VirtualTaobao generates virtual customers and interactions, enabling our agent to be tested in a simulated "live" environment.

\paragraph{Offline Evaluation}
For offline evaluation, we use the following benchmark datasets:

\begin{itemize}
    \item \textbf{MovieLens (100k\footnote{https://grouplens.org/datasets/movielens/100k/} and 1M\footnote{https://grouplens.org/datasets/movielens/1m/})}: These datasets, derived from the MovieLens website, feature user ratings of movies. The ratings are on a 5-star scale, with each user providing at least 20 ratings. Movies and users are characterized by 23 and 5 features, respectively.
    \item \textbf{Douban-Book\footnote{https://huggingface.co/datasets/larrylawl/douban-dushu}}: 
    The Douban-Book dataset is a collection of user interactions and book information derived from the Douban website, a popular Chinese social networking service. This dataset primarily focuses on user ratings of books. 
    \item \textbf{Book-Crossing\footnote{https://grouplens.org/datasets/book-crossing/}}: The Book-Crossing dataset originates from an online book club known for its book exchange and tracking services. This dataset encompasses user ratings, book information, and user-book interactions. 
\end{itemize}

\subsubsection{Baseline}
In our experiments, we employ the following algorithms as the baseline:

\begin{itemize}
    \item \textbf{Deep Deterministic Policy Gradient (DDPG)~\cite{lillicrap2015continuous}}: An off-policy method suitable for environments with continuous action spaces, employing a target policy network for action computation.
    \item \textbf{Soft Actor-Critic (SAC)~\cite{haarnoja2018soft}
}: An off-policy maximum entropy Deep RL approach, optimizing a stochastic policy with clipped double-Q method and entropy regularization.
    \item \textbf{Twin Delayed DDPG (TD3)~\cite{fujimoto2018addressing}}: An enhancement over DDPG, incorporating dual Q-functions, less frequent policy updates, and noise addition to target actions.
    \item \textbf{MACS~\cite{wang2023plug}} A method that introduces a  counterfactual synthesis policy into the RL-based recommender systems, generating the counterfactual user interaction based on the casual view of MDP for data augmentation. 
    \item \textbf{TPGR~\cite{chen2019large}}: A model for large-scale interactive recommendations, combining RL and a binary tree structure.
    \item \textbf{PGPR~\cite{xian2019reinforcement}}: An explainable recommendation model that integrates knowledge awareness with RL techniques.
    \item \textbf{DRR-att~\cite{LIU2020106170}}: A model for interactive recommender systems that introduces an attention network into the state representation module of a deep reinforcement learning recommendation framework. The DRR-att is implemented on the DDPG.
\end{itemize}

Note that the methods TPGR and PGPR use the knowledge graphs which are not available in our experiment, hence we removed the related part to conduct the experiments.

\subsubsection{Evaluation Measures}
The effectiveness of our model is assessed using different measures in online and offline environments:

\begin{itemize}
    \item \textbf{Online Evaluation:} VirtualTaobao uses click-through rate (CTR) as the indicator of effectiveness. The formula for CTR is as follows:
    \begin{align}
    \notag
        \text{CTR} = \frac{\text{episode\_return}}{\text{episode\_length} \times \text{maximum\_reward}} ,
    \end{align}
    where $maximum\_reward$ represents the highest potential reward obtainable in a single step within an episode.
    \item \textbf{Offline Evaluation:} For dataset evaluation, we utilize three widely recognized numerical criteria: Precision, Recall, and Accuracy. 
\end{itemize}

\begin{table*}[!h]
\centering
\caption{Performance comparisons of our method with baselines on the MovieLens datasets, Douban-Book, and BookCrossing datasets. The best results are highlighted in bold and the second best use the symbol *.}
\label{Offline}
\begin{tabular}{c|ccc|ccc}
\hline
\multicolumn{1}{c|}{\multirow{2}{*}{}} & \multicolumn{3}{c|}{MovieLens-100k}                                        & \multicolumn{3}{c}{MovieLens-1M}                                          \\
\multicolumn{1}{c|}{}                  & Recall               & Precision                   & Accuracy              & Recall               & Precision                   & Accuracy             \\ \hline
\multicolumn{1}{c|}{DDPG}              & 0.4611$\pm$0.0091        & 0.4182$\pm$0.0053               & 0.4512$\pm$0.0312         &  0.7440$\pm$0.0045*        & 0.4310$\pm$0.0023               & 0.5820$\pm$0.0028        \\
\multicolumn{1}{c|}{SAC}               & 0.6899$\pm$0.0015*         & 0.4991$\pm$0.0079                & 0.6266$\pm$0.0102*         & 0.7149$\pm$0.0438        & 0.4103$\pm$0.0086               & 0.7016$\pm$0.0557        \\
\multicolumn{1}{c|}{TD3}               & 0.6822$\pm$0.1352        & 0.5001$\pm$0.0043*               & 0.6234$\pm$0.0921           & 0.7034$\pm$0.0546        & 0.4201$\pm$0.0055               & 0.7121$\pm$0.0585*        \\
\multicolumn{1}{c|}{TPGR}              & 0.3758$\pm$0.0026        & 0.3242$\pm$0.0077               & 0.3698$\pm$0.0026         & 0.6889$\pm$0.0088        & 0.3827$\pm$0.0108               & 0.5023$\pm$0.0067        \\
\multicolumn{1}{c|}{PGPR}              & 0.4252$\pm$0.0047        & 0.3988$\pm$0.0043               & 0.4128$\pm$0.0088         & 0.6927$\pm$0.0091        & 0.4023$\pm$0.0044               & 0.5122$\pm$0.0088        \\
\multicolumn{1}{c|}{DRR-att}            & \multicolumn{1}{l}{0.5213$\pm$0.0090} & \multicolumn{1}{l}{0.4784$\pm$0.0099} & \multicolumn{1}{l|}{0.4824$\pm$0.0070} & \multicolumn{1}{l}{0.6251$\pm$0.0081} & \multicolumn{1}{l}{0.4322$\pm$0.0067*} & \multicolumn{1}{l}{0.5278$\pm$0.0129} \\
\multicolumn{1}{c|}{DDPG-CIDS}              &  \textbf{0.7233$\pm$0.0024}                &     \textbf{0.5241$\pm$0.0044}                       &                \textbf{0.6721$\pm$0.0029}       &     \textbf{0.7442$\pm$0.0078}                 &     \textbf{0.4521$\pm$0.0055}                        &  \textbf{0.7124$\pm$0.0299}                    \\ \hline
\end{tabular}

\begin{tabular}{c|ccc|ccc}
\hline
\multirow{2}{*}{} & \multicolumn{3}{c|}{Douban-Book}                                    & \multicolumn{3}{c}{BookCrossing}                                   \\
                  & Recall               & Precision            & Accuracy              & Recall               & Precision            & Accuracy             \\ \hline
DDPG              & 0.4611$\pm$0.0091        & 0.4182$\pm$0.0053        & 0.4512$\pm$0.0312         & 0.0744$\pm$0.0045*        & 0.0531$\pm$0.0023*        & 0.0582$\pm$0.0028        \\
SAC               & 0.4700$\pm$0.0023        & 0.4203$\pm$0.0039        & 0.4400$\pm$0.0029         & 0.0721$\pm$0.0029        & 0.0511$\pm$0.0039        & 0.0572$\pm$0.0022        \\
TD3               & 0.4711$\pm$0.0051        & 0.4199$\pm$0.0019        & 0.4488$\pm$0.0022         & 0.0702$\pm$0.0019        & 0.0510$\pm$0.0019        & 0.0566$\pm$0.0033        \\
TPGR              & 0.4555$\pm$0.0192        & 0.4112$\pm$0.0048        & 0.4488$\pm$0.0027         & 0.0725$\pm$0.0032        & 0.0452$\pm$0.0044        & 0.0640$\pm$0.0041        \\
PGPR              & 0.4602$\pm$0.0099        & 0.4182$\pm$0.0077        & 0.4460$\pm$0.0053         & 0.0699$\pm$0.0011        & 0.0393$\pm$0.0012        & 0.0733$\pm$0.0013*        \\
DRR-att           & 0.4827$\pm$0.0088* & 0.4213$\pm$0.0040* & 0.4521$\pm$0.0040* & 0.0702$\pm$0.0050 & 0.0423$\pm$0.0030 & 0.0587$\pm$0.0044 \\
DDPG-CIDS              &   \textbf{0.5422$\pm$0.0044}                  &     \textbf{0.4721$\pm$0.0041}                 &   \textbf{0.4817$\pm$0.0098}                    &    \textbf{0.0892$\pm$0.0056}                  &   \textbf{0.0723$\pm$0.0033}                   &   \textbf{0.0823$\pm$0.0044}          \\ \hline
\end{tabular}
\end{table*}

\subsection{Overall Results}
\paragraph{Online Simulator.}
The performance comparison, as seen in~\Cref{fig:CIDS_OVERALL} exhibits that algorithms enhanced with CIDS, namely DDPG-CIDS, SAC-CIDS, and TD3-CIDS, surpass the baseline methods significantly across the learning episodes.
DDPG-CIDS, for instance, indicates a considerable uplift in CTR over the standard DDPG algorithm. This improvement highlights the impact of integrating a causal understanding into the policy learning mechanism. SAC-CIDS similarly outperforms the conventional SAC model, which underscores the robustness of CIDS in environments with high-dimensional action spaces. The TD3-CIDS also shows substantial gains, especially in the later stages of the learning curve, suggesting the effectiveness of CIDS in enhancing temporal-difference learning algorithms. 
\Cref{fig:bar_overall} details the 1-step CTR performance of all baselines and our DDPG-CIDS(since the method DRR-att is also implemented on
the DDPG framework) in the VirtualTaobao simulation. The DDPG-CIDS again stands out, affirming its superior policy learning capacity by consistently achieving higher average CTRs compared to the other methods.

\begin{figure}[h]
    \centering
    \includegraphics[width=0.8\linewidth]{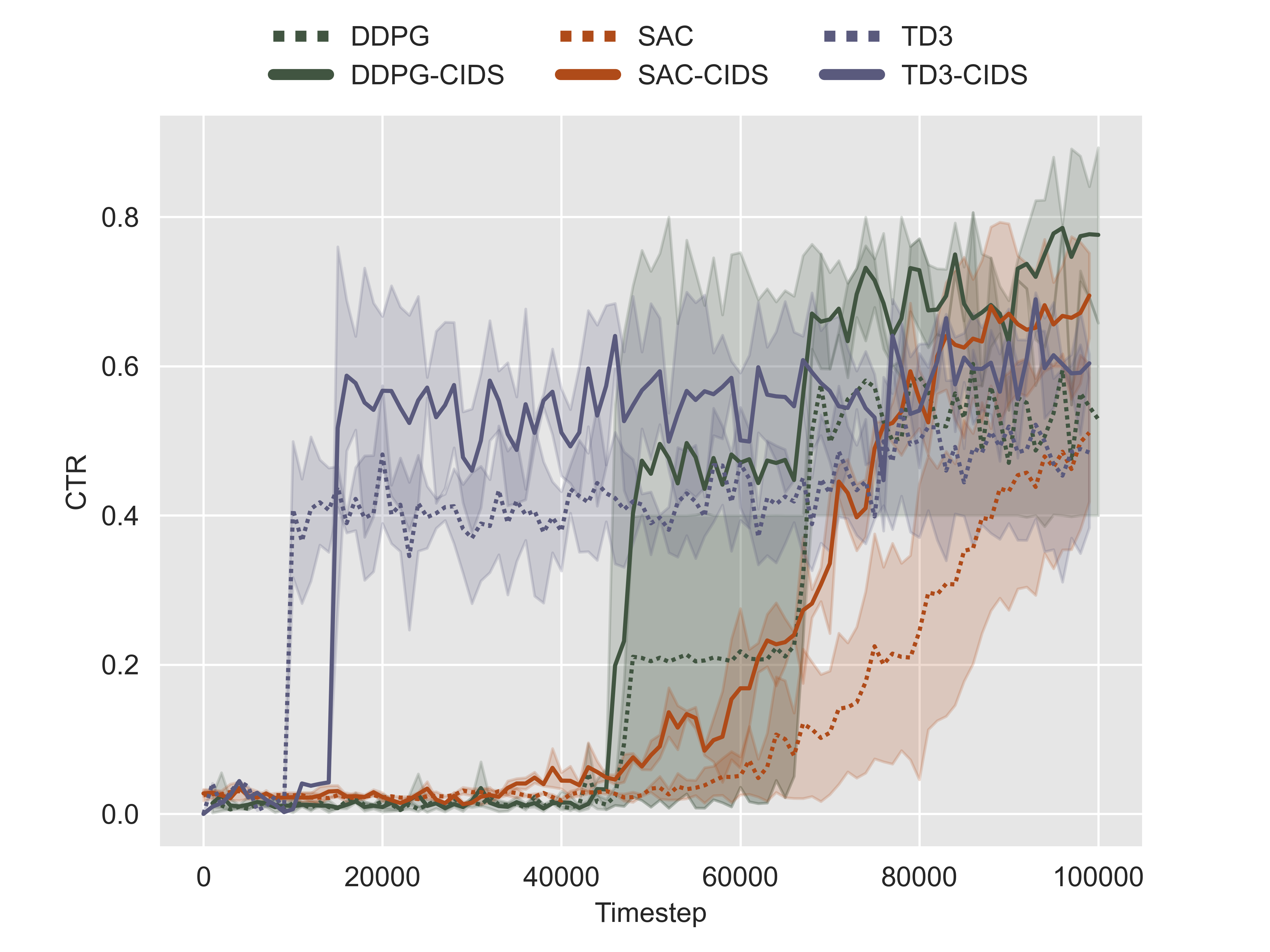}
\caption{Performance comparison of baseline algorithms and corresponding CIDS-enhanced methods in the VirtualTaobao simulation.}
\label{fig:CIDS_OVERALL}
\end{figure}

\paragraph{Offline Dataset.}
\label{subsec:offline_performance}
Our experimental evaluation showcases the efficacy of the DDPG-CIDS framework across several datasets, including MovieLens-100k, MovieLens-1M, Douban-Book, and BookCrossing. We benchmarked against a range of state-of-the-art algorithms and highlighted the best and second-best results with bold and asterisks, respectively, as detailed in Table~\ref{Offline}.
Particularly on MovieLens-100k, DDPG-CIDS surpassed all baselines, reinforcing the benefits of integrating CIDS into the recommendation mechanism. Its top scores in Recall, Precision, and Accuracy highlight the method's effectiveness. On the larger MovieLens-1M, DDPG-CIDS maintained top-tier results, suggesting scalability.
DDPG-CIDS's superior performance on the Douban-Book and BookCrossing datasets confirmed its adaptability to various content types and user interactions.

Overall, these results clearly manifest the advantages of incorporating CIDS into the recommendation policy learning framework. The boosted performance in Precision and Recall signifies that CIDS facilitates the identification of more relevant information, while the improved Accuracy reflects the overall reliability of the policy. 

\begin{figure}[h]
    \centering
    \includegraphics[width=0.7\linewidth]{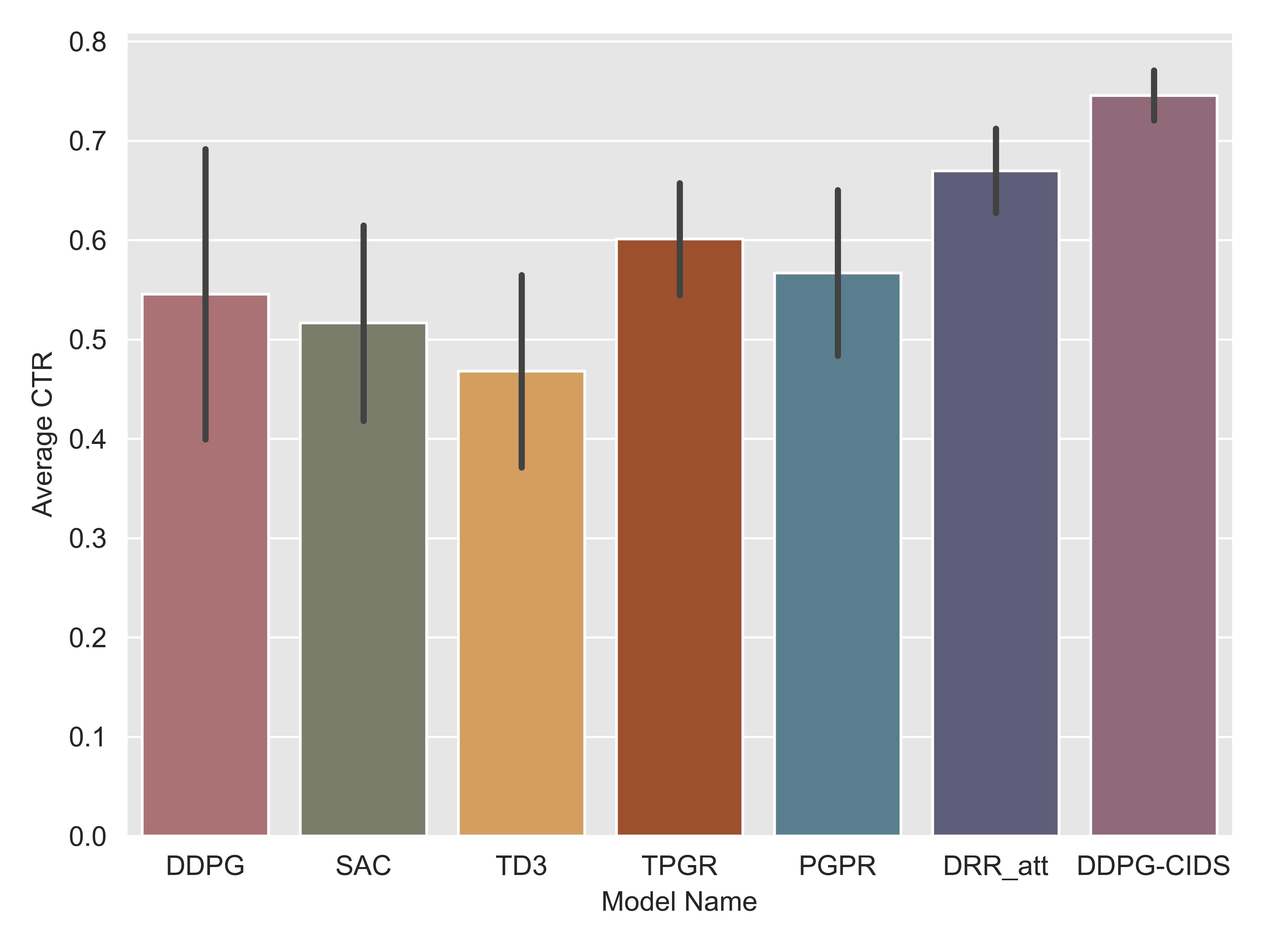}
\caption{The 1-step CTR performance in the VirtualTaobao simulation is presented as the mean with error bars.}
\label{fig:bar_overall}
\end{figure}

\begin{figure*}[h]
     \centering
     \begin{subfigure}[b]{0.3\linewidth}
         \centering
         \includegraphics[width=\linewidth]{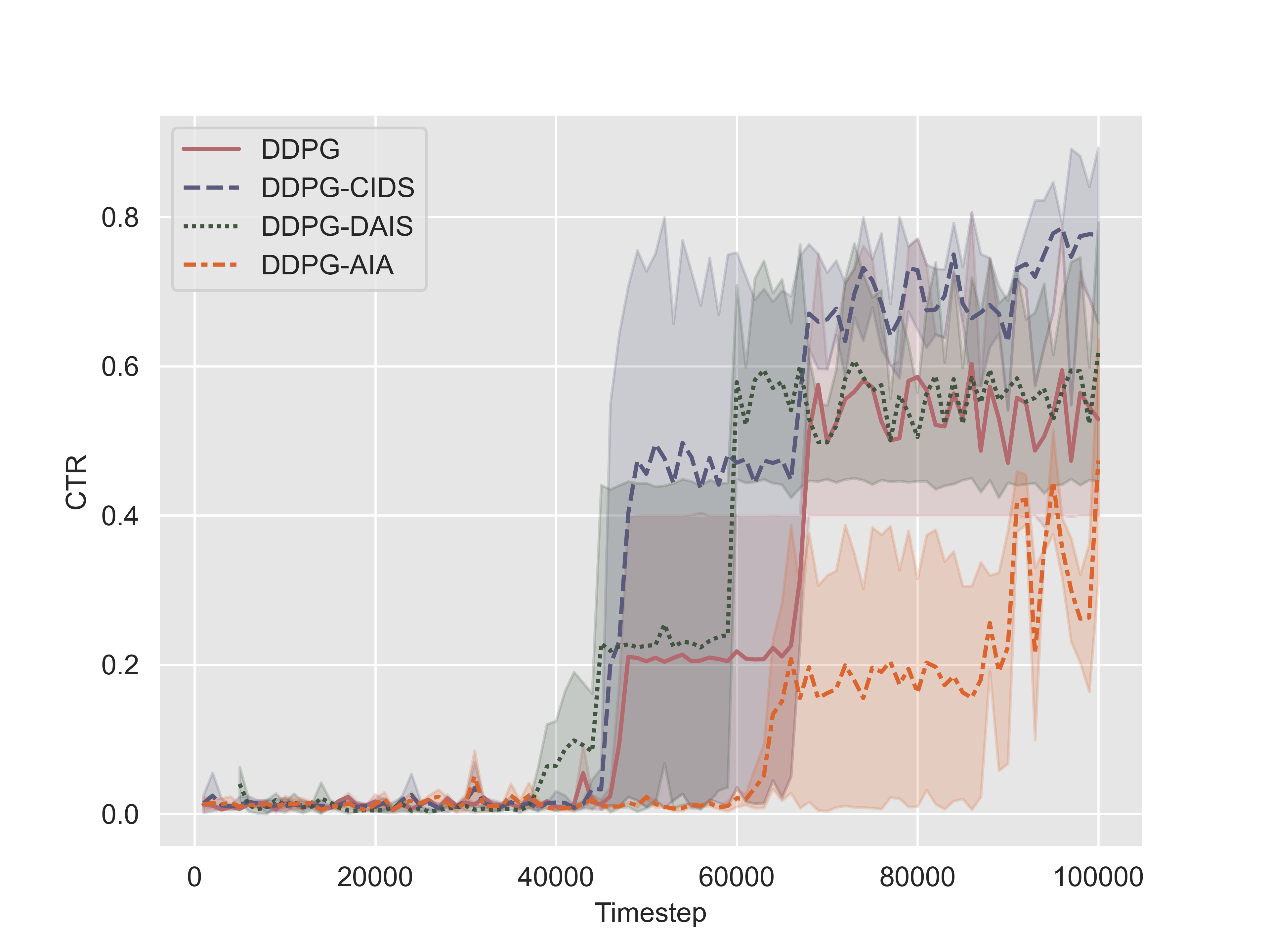}
         \caption{}
         \label{a}
     \end{subfigure}
     \begin{subfigure}[b]{0.3\linewidth}
         \centering
         \includegraphics[width=\linewidth]{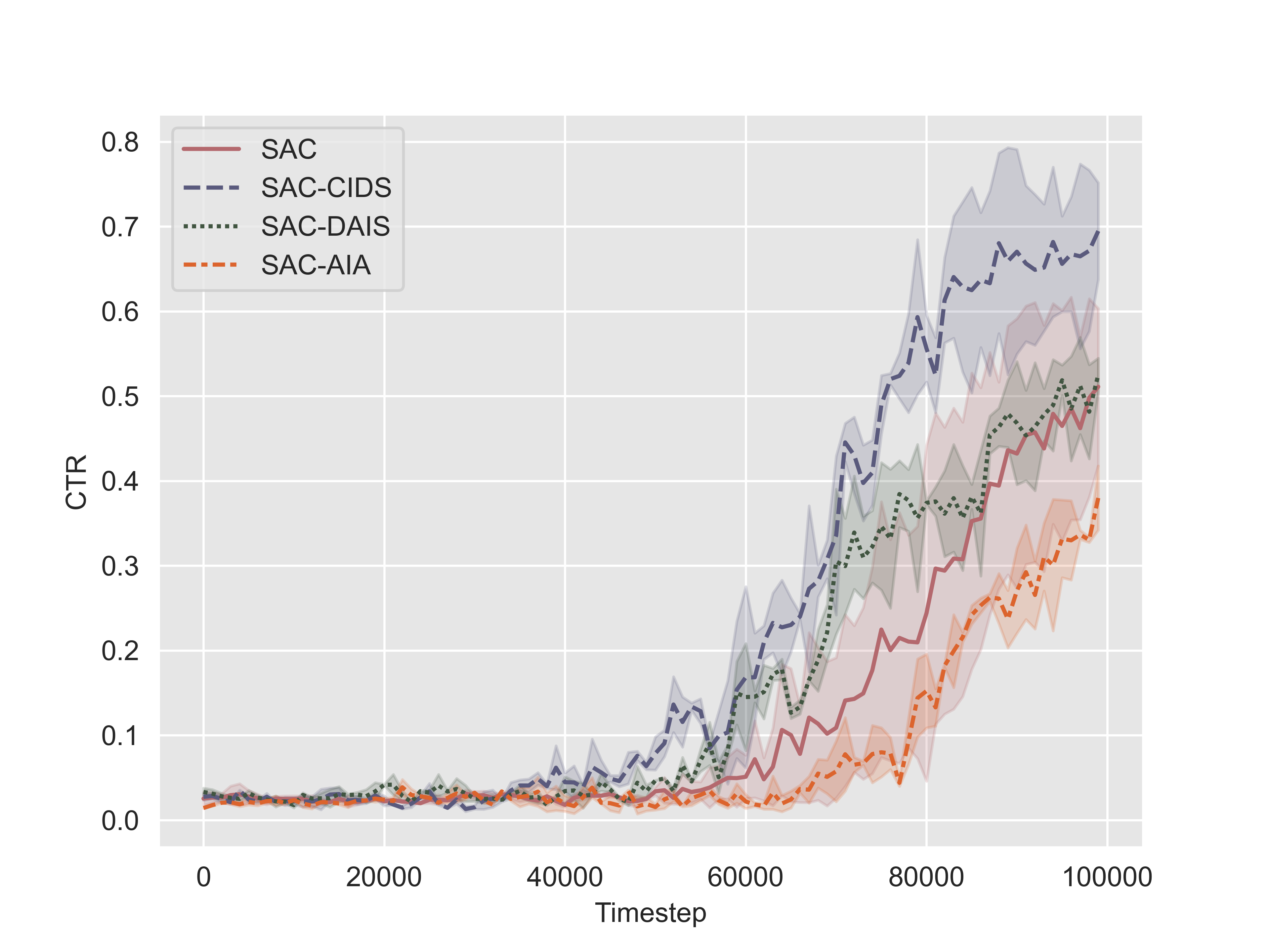}
         \caption{}
         \label{b}
     \end{subfigure}
     \begin{subfigure}[b]{0.3\linewidth}
         \centering
         \includegraphics[width=\linewidth]{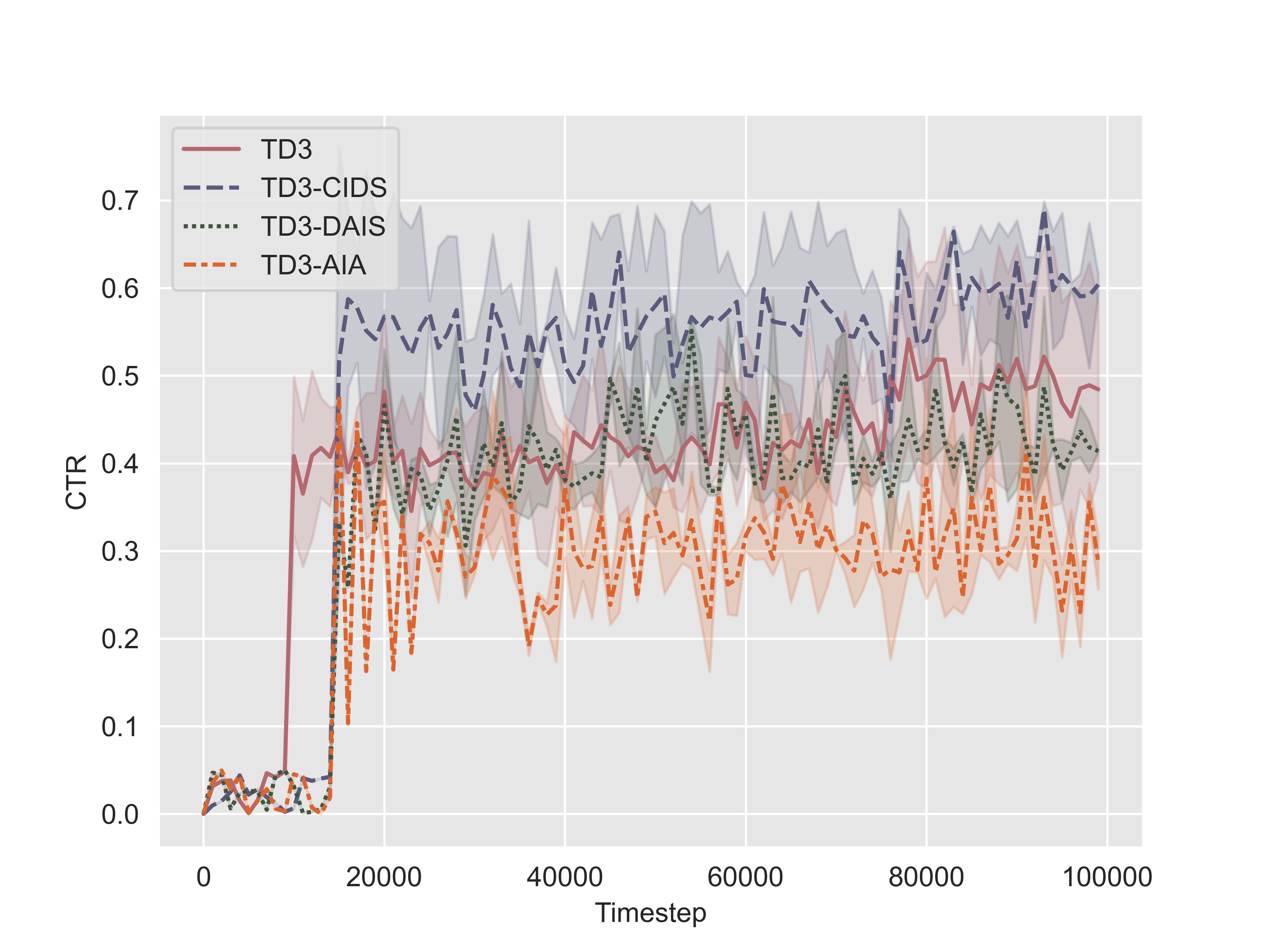}
         \caption{}
         \label{c}
     \end{subfigure}
        \caption{Evaluation with different RL frameworks: (a) DDPG as the backbone, (b) SAC as the backbone, and (c) TD3 as the backbone. Ablation versions with only DAIS representation and AIA representation are also included in each backbone.}
\label{fig:ablation_study}
\end{figure*}

\subsection{Ablation Study}
\label{subsec:ablation_study}
In this section, we first delve into the ablation study designed to dissect the contributions of the different components of our framework, specifically focusing on CIDS and its constituents: DAIS and AIA. Then we investigate the impact of the sparsity of the learned causal structure on recommendation policy learning.

\subsubsection{Impact of CIDS, DAIS, and AIA on Recommendation Policy Learning}
\label{subsubsec:impact_cids_dais_aia}

Our analysis is aimed at understanding the individual and combined effects of DAIS and AIA components on the policy learning process. We considered three variants: DDPG enhanced with only DAIS (DDPG-DAIS), only AIA (DDPG-AIA), and both DAIS and AIA (DDPG-CIDS). As depicted in ~\Cref{fig:ablation_study}, DDPG-CIDS consistently outperforms the other variants, suggesting that the combination of DAIS and AIA is critical for capturing the complete causal structure necessary for robust policy learning. 
While the DDPG-DAIS variant demonstrates some improvements over the baseline DDPG, the DDPG-AIA variant does not perform as well, even falling behind the baseline performance. This observation can be attributed to the inherent characteristics of the state representations. DAIS focuses on the aspects of the state that are directly affected by the action, often containing the most important information needed for making decisions. In contrast, AIA represents state dimensions that do not directly interact with action variables but influence DAIS indirectly. Training with only the AIA representation reveals that a state representation lacking direct action-relevant information can detrimentally affect the learning of the recommendation policy. Hence, it is the combined effect of DAIS and AIA within the DDPG-CIDS framework that leads to superior recommendation performance, underscoring the importance of integrating comprehensive causal knowledge into the recommendation process.

\subsubsection{Results with Different RL Frameworks}
\label{subsubsec:rl_frameworks_results}
Further, we extended our ablation study to evaluate the impact of integrating CIDS into different RL frameworks: DDPG, SAC, and TD3. As illustrated in~\Cref{fig:CIDS_OVERALL} and~\Cref{fig:ablation_study}, the inclusion of CIDS enhances the performance of all backbone architectures. Notably, SAC-CIDS and TD3-CIDS demonstrate a marked improvement in CTR, highlighting our method's flexibility with different reinforcement learning algorithms. The consistent improvement boost across diverse frameworks validates the adaptability and effectiveness of the CIDS framework for RL-based recommender systems.

\begin{table}[h]
\centering
\caption{Mean and standard deviation of CTR across various $\lambda_1$ values over different timestep $t$ in VirtualTaobao, presented as percentages.}
\label{tab:lamda_results}
\resizebox{0.98\linewidth}{!}{%
\begin{tabular}{c|cccc}
\hline
$t$   & $\lambda_1=0$     & $\lambda_1=1\mathrm{e}{-4}$         & $\lambda_1=5\mathrm{e}{-4}$           & $\lambda_1=9\mathrm{e}{-4}$  \\ \hline
$2\mathrm{e}4$ & 1.21 $\pm$ 0.17   & \textbf{2.46 $\pm$ 0.64}  & 1.53 $\pm$ 0.68            & 1.24 $\pm$ 0.17   \\
$4\mathrm{e}4$ & 1.32 $\pm$ 0.07   & \textbf{2.91 $\pm$ 0.54}  & 2.52 $\pm$ 0.18            & 1.44 $\pm$ 0.24   \\
$6\mathrm{e}4$ & 24.66 $\pm$ 18.37 & 42.81 $\pm$ 27.55        & \textbf{58.38 $\pm$ 8.01}  & 23.50 $\pm$ 17.76 \\
$8\mathrm{e}4$ & 43.38 $\pm$ 4.77  & 58.75$\pm$ 15.39         & \textbf{71.21 $\pm$ 12.73} & 37.05 $\pm$10.72  \\
$1\mathrm{e}5$ & 53.55 $\pm$ 19.29 & 67.99 $\pm$ 12.45       & \textbf{77.87 $\pm$ 4.18}  & 51.48 $\pm$ 16.84 \\ \hline
\end{tabular}
}
\end{table}
\subsubsection{Impact of the Sparsity of Learned Causal Structure on Recommendation Policy Learning}
\label{subsec:impact_sparsity}

The sparsity of the structural matrices in our model is regulated by the hyperparameters $\lambda_1$ and $\lambda_2$, as described in Section~\ref{sec:obj}. A greater value for these parameters induces a sparser causal structure. Specifically, we examine the role of $\lambda_1$, which influences the causal structure between actions and state dimensions, to assess its impact on the performance of the recommendation policy. This focus stems from observations in Section~\ref{subsubsec:impact_cids_dais_aia}, which highlighted that the action-state causal structure significantly affects the recommendation policy, more so than the inter-state causal relationships.

Table~\ref{tab:lamda_results} summarizes the mean and standard deviation of the CTR across a spectrum of $\lambda_1$ values, tracked over successive timesteps within the VirtualTaobao simulation. It becomes apparent from the data that the sparsity level, governed by $\lambda_1$, is a critical determinant in the policy's learning performance. Optimal sparsity can enhance policy learning, but an overly sparse structure (such as when $\lambda_1$ is set to 9e-4) can be detrimental, possibly due to insufficient information for the policy to leverage. Conversely, a lower $\lambda_1$ value may facilitate a more robust early learning phase by providing a richer informational context. However, too low a $\lambda_1$ value may introduce noise through non-essential dimensions, thus impeding the policy's optimization process.

\section{Related Work}
\vspace{1mm}\noindent\textbf{DRL-based recommender system.}

DRL-based recommender systems model the interaction recommendation process as Markov Decision Processes (MDPs), utilizing deep learning to estimate the value function and tackle high-dimensional MDPs~\cite{mahmood2007learning,CHEN2023110335}. Recognizing the significance of negative feedback in understanding user preferences, \citet{zhao2018recommendations} introduced DEERS, which processes positive and negative signals separately at the input layer to avoid negative feedback overwhelming positive signals due to their sparsity. \citet{chen2020knowledge} incorporated knowledge graphs into DRL for interactive recommendation, employing a local knowledge network to enhance efficiency. \citet{hong2020nonintrusive} proposed NRRS, a model-based approach integrating nonintrusive sensing and reinforcement learning for personalized dynamic music recommendation. NRRS trains a user reward model that derives rewards from three user feedback sources: scores, opinions, and wireless signals. Moving beyond predefining a reward function, \citet{chen2021generative} introduced InvRec, utilizing inverse reinforcement learning to infer a reward function from user behaviors and directly learning the recommendation policy from these behaviors. InvRec employs inverse DRL as a generator to augment state-action pairs, offering a novel approach to the task. Addressing offline RL methods, \citet{wang2023causal} proposed CDT4Rec, which designed a new causal mechanism to estimate the reward function. Finally, \citet{chen2023intrinsically} proposed a causal augmentation method for RLRS, providing a new perspective to address the exploration problem in RLRS.

\vspace{1mm}\noindent\textbf{Causal Recommendation.}
In recent years, the recommendation domain has witnessed significant advancements through the integration of causal inference techniques. These techniques, especially in de-biasing training data, have been transformative for the field.
~\citet{bonner2018causal} developed a domain adaptation algorithm, capitalizing on biased logged feedback to predict randomized treatment effects and address the challenge of random exposure.
Further expanding this field,~\citet{liu2020general} presented KDCRec, a knowledge distillation framework aimed at addressing bias in recommender systems by extracting insights from uniformly distributed data.~\citet{zhang2021causal} tackled the pervasive issue of popularity bias, devising a novel causal inference paradigm to adjust recommendation scores through targeted causal interventions.
Additionally,~\citet{9737000} proposed CI-LightGCN, a Causal Incremental Graph Convolution method for updating graph convolutional networks in recommender systems, efficiently handling model updates with new data while maintaining recommendation accuracy.
The application of counterfactual inference in recommender systems has also gained traction, being utilized for path-specific effects removal~\cite{wang2021clicks} and out-of-distribution (OOD) generalization~\cite{wang2022causal}. Moreover, an increasing number of studies are adopting counterfactual reasoning for objectives such as providing explanations, enhancing model interpretability, and learning robust representations~\cite{tan2021counterfactual, madumal2020explainable, zhang2021causerec}.

\section{Conclusion}
In this study, we address the challenge of high-dimensional and noisy state spaces in RLRS by introducing Causal-Indispensable State Representations (CIDS). Focusing on Directly Action-Influenced State Variables (DAIS) and Action-Influence Ancestors (AIA), our approach identifies key state components essential for effective recommendation policy learning. Utilizing conditional mutual information, CIDS effectively discerns causal relationships within the generative process, isolating crucial state variables. Theoretical evidence supports the identifiability of these variables, allowing for the construction of optimized state representations. This novel framework enables training on a refined subset of the state space, significantly enhancing recommendation accuracy and efficiency, as demonstrated in our extensive experimental evaluations.
Future work could focus on incorporating potential confounders to further refine and strengthen the causal relationships within RLRS. Exploring the simultaneous training of both the representation and policy also presents a promising direction.

\appendix
% \chapter{Appendix for the paper " "}
\section{Definitions in Causality}
Here we briefly mention some fundamental definitions~\cite{10.5555/3202377, pearl2009causality, spirtes2000causation}, which we use in our paper to present and prove our methodology.

\begin{definition}[d-Separation~\cite{10.5555/3202377}]
\label{def-d}
In  Directed Acyclic Graph (DAG) $\mathcal{G}$, a path between two nodes, denoted as $i_n$ and $i_m$, is considered blocked by a set $\mathbf{S}$. This occurs when neither $i_n$ nor $i_m$ are included in $\mathbf{S}$, and there exists a node $i_k$ that satisfies one of the following two possibilities:
\begin{enumerate}[(i)]
\item $i_k \in S$ and $i_{k-1} \rightarrow i_k \rightarrow i_{k+1}$, $i_{k-1} \leftarrow i_k \leftarrow i_{k+1}$, or $i_{k-1} \leftarrow i_k \rightarrow i_{k+1}$; 
\item $i_k$ and its descendants are not part of the blocking set $\mathbf{S}$, and $i_{k-1} \rightarrow i_k \leftarrow i_{k+1}$.
\end{enumerate}
\end{definition}

\begin{definition}[Structural Causal Models~\cite{pearl2009causality}]
A Structural Causal Model (SCM) $\mathcal{M} := (S, P_U)$ is associated with a DAG \( \mathcal{G} \), which consists of a collection $S$ of $n$ structural assignments :

\begin{equation}
\begin{aligned}
X_i := f_i(\boldsymbol{PA}_i, U_i), \:\: i = 1,...,n, 
\end{aligned}
\end{equation}
where $V = \{X_1,...,X_n\}$ is a set of endogenous variables, and $\boldsymbol{PA}_i \subseteq \{X_1,...,X_n\} \setminus \{X_i\}$ represent parents of $X_i$, which are also called direct causes of $X_i$. And $U = \{U_1,...,U_n\}$ are noise variables, determined by unobserved factors. We assume that noise variables are jointly independent. Correspondingly, $P_U$ is the joint distribution over the noise variables. Each structural equation $f_i$ is a causal mechanism that determines the value of $X_i$ based on the values of $\boldsymbol{PA}_i$ and the noise term $U_i$.
\end{definition}

\begin{definition}[Markov property~\cite{10.5555/3202377}] Given a DAG $\mathcal{G}$, and a joint distribution $P_X$, the distribution is said to satisfy
\begin{enumerate}[(i)]
    \item global Markov property in relation to the DAG $\mathcal{G}$ if the condition $\mathbf{A} \independent_G \mathbf{B} | \mathbf{S}$ implies that $\mathbf{A} \independent \mathbf{B} | \mathbf{S}$ holds true for all distinct sets of vertices $\mathbf{A}, \mathbf{B}, \mathbf{S}$.\ \label{global_M}
    \item Markov factorization property in relation to the DAG $\mathcal{G}$ if the expression $p(\mathbf{x}) = p(x_1, x_2, ..., x_d) = \prod_{i=1}^d p(x_j|\textbf{pa}^\mathcal{G}_j)$ holds true.
\end{enumerate}\label{Markov}
\end{definition}

\begin{definition}[Causal Faithfulness~\cite{10.5555/3202377, pearl2009causality, spirtes2000causation}]
    A distribution $P$ is faithful to a DAG  $\mathcal{G}$ if no conditional independence relations other than the ones entailed by the Markov property are present.
\end{definition}

\section{Proof of Theorem 3.3}
\label{app:th3.3}
\begin{proof}

Firstly, the Markov condition ensures that for any pair of non-adjacent variables in the causal graph \( \mathcal{G} = (\mathbf{V}, \mathcal{E}) \), denoted \( V_i \) and \( V_j \), there is conditional independence between them given a set of other variables. This condition allows us to infer the absence of direct causal links between certain variables in the MDP.

Secondly, the faithfulness assumption posits that all observed conditional independencies are indicative of true separations in the causal graph. This implies that if two variables are found to be conditionally independent given a set of other variables, there is indeed no direct causal path between them in the graph.

Together, these conditions permit the identification of the binary structural masks \( \matr{M}_{s \to s} \) and \( \matr{M}_{a \to s} \)
defined over the set \( \mathbf{V} \), which are identifiable through the examination of conditional independence relationships.
\end{proof}

\section{Objective Function}
\label{app:obj}
The objective function for the DAIS predictive model is given by:
\begin{equation}
\begin{aligned}
    & \mathcal{L}_{\text{DAIS-model}} = -\mathcal{L}_{\text{DAIS}} + \lambda_1 \| \matr{M}_{a \to s} \|_1 \\
    & = \Bigg(\mathbb{E}_{s_t, a_t, s_{t+1}\sim
    \mathcal{D}} \left[ \sum_{j=1}^d \log P(s^j_{t+1} \in \text{DAIS}_{t+1} | s_t, \matr{M}^{(,j)}_{a \to s} ; \theta_{D_1}) \right] \\
    & -\mathbb{E}_{s_t, a_t, s_{t+1}\sim
    \mathcal{D}} \left[ \sum_{j=1}^d \log P(s^j_{t+1} \in \text{DAIS}_{t+1} | a_{t}, s_t, \matr{M}^{(,j)}_{a \to s} ; \theta_{D_2}) \right]
    \Bigg)\\ 
    &+ \lambda_1 \| \matr{M}_{a \to s} \|_1,
    \end{aligned}
\end{equation}
where $\mathcal{L}_{\text{DAIS}}$ is the loss related to the learning of DAIS and $\lambda_1$ is a hyperparameter that controls the sparsity of the action-to-state causal structure matrix.

For the AIA predictive model, the objective function is:
\begin{equation}
\begin{aligned}
    & \mathcal{L}_{\text{AIA-model}} = -\mathcal{L}_{\text{AIA}} + \lambda_2 \| \matr{M}_{s \to s} \|_1\\
     & + \Bigg( \mathbb{E}_{s_{t-1}, a_{t-1}, s_{t} \sim \mathcal{D}} \left[ \sum_{j=1}^d \log P(s^j_{t-1} \in \text{AIA}_{t-1} | s_t, \matr{M}^{(,j)}_{a \to s}; \theta_{A_1}) \right] \\
& - \mathbb{E}_{s_{t-1}, a_{t-1}, s_{t} \sim \mathcal{D}} \left[ \sum_{j=1}^d \log P(s^j_{t-1} \in \text{AIA}_{t-1}  | a_{t-1}, s_t, \matr{M}^{(,j)}_{a \to s}; \theta_{A_2}) \right] \Bigg)\\
& \lambda_2 \| \matr{M}_{s \to s} \|_1,\\
\end{aligned}
\end{equation}
where $\mathcal{L}_{\text{AIA}}$ is the loss associated with the learning of AIA and $\lambda_2$ is a hyperparameter for the sparsity of the state-to-state causal structure matrix.

\section{Implementation Details}
\label{app:imp}
Initially, we train a sub-optimal recommendation policy utilizing the DDPG algorithm with default parameters as delineated in \cite{lillicrap2015continuous}. This policy undergoes training across 1,000,000 episodes, with the best-performing iteration saved as our expert policy. Utilizing this expert policy, we generate a dataset of expert trajectories within the environment. It should be noted that the expert's exposure to the environment is limited, with only a finite number of trajectories being sampled. These expert trajectories comprise the observed dataset employed for training the DAIS and AIA components.

For the DAIS and AIA predictive models, we opt for Multilayer Perceptron (MLP) networks. Both models feature a tri-layered fully connected architecture with each layer outputting 128 units. Following each hidden layer is a ReLU activation function. The hyperparameters $\lambda_i$ are set as follows: $\lambda_1 = 5 \times 10^{-4}$ and $\lambda_2 = 10^{-4}$.
In training the recommendation policy, we fix the actor network's learning rate at $10^{-4}$ and the critic network's at $10^{-3}$. The discount factor, denoted by $\gamma$, is established at 0.95, paired with a soft target update rate, $\tau$, of 0.001. The network's hidden layer size is determined to be 128, and the replay buffer capacity is $10^6$ entries.

We adhere to parameter configurations as specified in stable baselines3\footnote{https://stable-baselines3.readthedocs.io/en/master/} or as originally reported in the respective baseline papers for all baseline methods.

\bibliographystyle{ACM-Reference-Format}
\balance
\bibliography{sample-base}

%%% -*-BibTeX-*-
%%% Do NOT edit. File created by BibTeX with style
%%% ACM-Reference-Format-Journals [18-Jan-2012].

\begin{thebibliography}{40}

%%% ====================================================================
%%% NOTE TO THE USER: you can override these defaults by providing
%%% customized versions of any of these macros before the \bibliography
%%% command.  Each of them MUST provide its own final punctuation,
%%% except for \shownote{}, \showDOI{}, and \showURL{}.  The latter two
%%% do not use final punctuation, in order to avoid confusing it with
%%% the Web address.
%%%
%%% To suppress output of a particular field, define its macro to expand
%%% to an empty string, or better, \unskip, like this:
%%%
%%% \newcommand{\showDOI}[1]{\unskip}   % LaTeX syntax
%%%
%%% \def \showDOI #1{\unskip}           % plain TeX syntax
%%%
%%% ====================================================================

\ifx \showCODEN    \undefined \def \showCODEN     #1{\unskip}     \fi
\ifx \showDOI      \undefined \def \showDOI       #1{#1}\fi
\ifx \showISBNx    \undefined \def \showISBNx     #1{\unskip}     \fi
\ifx \showISBNxiii \undefined \def \showISBNxiii  #1{\unskip}     \fi
\ifx \showISSN     \undefined \def \showISSN      #1{\unskip}     \fi
\ifx \showLCCN     \undefined \def \showLCCN      #1{\unskip}     \fi
\ifx \shownote     \undefined \def \shownote      #1{#1}          \fi
\ifx \showarticletitle \undefined \def \showarticletitle #1{#1}   \fi
\ifx \showURL      \undefined \def \showURL       {\relax}        \fi
% The following commands are used for tagged output and should be
% invisible to TeX
\providecommand\bibfield[2]{#2}
\providecommand\bibinfo[2]{#2}
\providecommand\natexlab[1]{#1}
\providecommand\showeprint[2][]{arXiv:#2}

\bibitem[Afsar et~al\mbox{.}(2022)]%
        {afsar2022reinforcement}
\bibfield{author}{\bibinfo{person}{M~Mehdi Afsar}, \bibinfo{person}{Trafford Crump}, {and} \bibinfo{person}{Behrouz Far}.} \bibinfo{year}{2022}\natexlab{}.
\newblock \showarticletitle{Reinforcement learning based recommender systems: A survey}.
\newblock \bibinfo{journal}{\emph{Comput. Surveys}} \bibinfo{volume}{55}, \bibinfo{number}{7} (\bibinfo{year}{2022}), \bibinfo{pages}{1--38}.
\newblock


\bibitem[Bonner and Vasile(2018)]%
        {bonner2018causal}
\bibfield{author}{\bibinfo{person}{Stephen Bonner} {and} \bibinfo{person}{Flavian Vasile}.} \bibinfo{year}{2018}\natexlab{}.
\newblock \showarticletitle{Causal embeddings for recommendation}. In \bibinfo{booktitle}{\emph{Proceedings of the 12th ACM conference on recommender systems}}. \bibinfo{pages}{104--112}.
\newblock


\bibitem[Chen et~al\mbox{.}(2019)]%
        {chen2019large}
\bibfield{author}{\bibinfo{person}{Haokun Chen}, \bibinfo{person}{Xinyi Dai}, \bibinfo{person}{Han Cai}, \bibinfo{person}{Weinan Zhang}, \bibinfo{person}{Xuejian Wang}, \bibinfo{person}{Ruiming Tang}, \bibinfo{person}{Yuzhou Zhang}, {and} \bibinfo{person}{Yong Yu}.} \bibinfo{year}{2019}\natexlab{}.
\newblock \showarticletitle{Large-scale interactive recommendation with tree-structured policy gradient}. In \bibinfo{booktitle}{\emph{Proceedings of the AAAI Conference on Artificial Intelligence}}, Vol.~\bibinfo{volume}{33}. \bibinfo{pages}{3312--3320}.
\newblock


\bibitem[Chen et~al\mbox{.}(2018)]%
        {chen2018stabilizing}
\bibfield{author}{\bibinfo{person}{Shi-Yong Chen}, \bibinfo{person}{Yang Yu}, \bibinfo{person}{Qing Da}, \bibinfo{person}{Jun Tan}, \bibinfo{person}{Hai-Kuan Huang}, {and} \bibinfo{person}{Hai-Hong Tang}.} \bibinfo{year}{2018}\natexlab{}.
\newblock \showarticletitle{Stabilizing reinforcement learning in dynamic environment with application to online recommendation}. In \bibinfo{booktitle}{\emph{Proceedings of the 24th ACM SIGKDD International Conference on Knowledge Discovery \& Data Mining}}. \bibinfo{pages}{1187--1196}.
\newblock


\bibitem[Chen et~al\mbox{.}(2020)]%
        {chen2020knowledge}
\bibfield{author}{\bibinfo{person}{Xiaocong Chen}, \bibinfo{person}{Chaoran Huang}, \bibinfo{person}{Lina Yao}, \bibinfo{person}{Xianzhi Wang}, \bibinfo{person}{Wenjie Zhang}, {et~al\mbox{.}}} \bibinfo{year}{2020}\natexlab{}.
\newblock \showarticletitle{Knowledge-guided deep reinforcement learning for interactive recommendation}. In \bibinfo{booktitle}{\emph{2020 International Joint Conference on Neural Networks (IJCNN)}}. IEEE, \bibinfo{pages}{1--8}.
\newblock


\bibitem[Chen et~al\mbox{.}(2023a)]%
        {chen2023intrinsically}
\bibfield{author}{\bibinfo{person}{Xiaocong Chen}, \bibinfo{person}{Siyu Wang}, \bibinfo{person}{Lianyong Qi}, \bibinfo{person}{Yong Li}, {and} \bibinfo{person}{Lina Yao}.} \bibinfo{year}{2023}\natexlab{a}.
\newblock \showarticletitle{Intrinsically motivated reinforcement learning based recommendation with counterfactual data augmentation}.
\newblock \bibinfo{journal}{\emph{World Wide Web}} \bibinfo{volume}{26}, \bibinfo{number}{5} (\bibinfo{year}{2023}), \bibinfo{pages}{3253--3274}.
\newblock


\bibitem[Chen et~al\mbox{.}(2024)]%
        {chen2024maximum}
\bibfield{author}{\bibinfo{person}{Xiaocong Chen}, \bibinfo{person}{Siyu Wang}, {and} \bibinfo{person}{Lina Yao}.} \bibinfo{year}{2024}\natexlab{}.
\newblock \showarticletitle{Maximum-Entropy Regularized Decision Transformer with Reward Relabelling for Dynamic Recommendation}.
\newblock \bibinfo{journal}{\emph{arXiv preprint arXiv:2406.00725}} (\bibinfo{year}{2024}).
\newblock


\bibitem[Chen et~al\mbox{.}(2023b)]%
        {CHEN2023110335}
\bibfield{author}{\bibinfo{person}{Xiaocong Chen}, \bibinfo{person}{Lina Yao}, \bibinfo{person}{Julian McAuley}, \bibinfo{person}{Guanglin Zhou}, {and} \bibinfo{person}{Xianzhi Wang}.} \bibinfo{year}{2023}\natexlab{b}.
\newblock \showarticletitle{Deep reinforcement learning in recommender systems: A survey and new perspectives}.
\newblock \bibinfo{journal}{\emph{Knowledge-Based Systems}}  \bibinfo{volume}{264} (\bibinfo{year}{2023}), \bibinfo{pages}{110335}.
\newblock
\showISSN{0950-7051}
\urldef\tempurl%
\url{https://doi.org/10.1016/j.knosys.2023.110335}
\showDOI{\tempurl}


\bibitem[Chen et~al\mbox{.}(2021)]%
        {chen2021generative}
\bibfield{author}{\bibinfo{person}{Xiaocong Chen}, \bibinfo{person}{Lina Yao}, \bibinfo{person}{Aixin Sun}, \bibinfo{person}{Xianzhi Wang}, \bibinfo{person}{Xiwei Xu}, {and} \bibinfo{person}{Liming Zhu}.} \bibinfo{year}{2021}\natexlab{}.
\newblock \showarticletitle{Generative inverse deep reinforcement learning for online recommendation}. In \bibinfo{booktitle}{\emph{Proceedings of the 30th ACM International Conference on Information \& Knowledge Management}}. \bibinfo{pages}{201--210}.
\newblock


\bibitem[Ding et~al\mbox{.}(2022)]%
        {9737000}
\bibfield{author}{\bibinfo{person}{Sihao Ding}, \bibinfo{person}{Fuli Feng}, \bibinfo{person}{Xiangnan He}, \bibinfo{person}{Yong Liao}, \bibinfo{person}{Jun Shi}, {and} \bibinfo{person}{Yongdong Zhang}.} \bibinfo{year}{2022}\natexlab{}.
\newblock \showarticletitle{Causal Incremental Graph Convolution for Recommender System Retraining}.
\newblock \bibinfo{journal}{\emph{IEEE Transactions on Neural Networks and Learning Systems}} (\bibinfo{year}{2022}), \bibinfo{pages}{1--11}.
\newblock
\urldef\tempurl%
\url{https://doi.org/10.1109/TNNLS.2022.3156066}
\showDOI{\tempurl}


\bibitem[Fujimoto et~al\mbox{.}(2018)]%
        {fujimoto2018addressing}
\bibfield{author}{\bibinfo{person}{Scott Fujimoto}, \bibinfo{person}{Herke Hoof}, {and} \bibinfo{person}{David Meger}.} \bibinfo{year}{2018}\natexlab{}.
\newblock \showarticletitle{Addressing function approximation error in actor-critic methods}. In \bibinfo{booktitle}{\emph{International conference on machine learning}}. PMLR, \bibinfo{pages}{1587--1596}.
\newblock


\bibitem[Gelada et~al\mbox{.}(2019)]%
        {gelada2019deepmdp}
\bibfield{author}{\bibinfo{person}{Carles Gelada}, \bibinfo{person}{Saurabh Kumar}, \bibinfo{person}{Jacob Buckman}, \bibinfo{person}{Ofir Nachum}, {and} \bibinfo{person}{Marc~G Bellemare}.} \bibinfo{year}{2019}\natexlab{}.
\newblock \showarticletitle{Deepmdp: Learning continuous latent space models for representation learning}. In \bibinfo{booktitle}{\emph{International Conference on Machine Learning}}. PMLR, \bibinfo{pages}{2170--2179}.
\newblock


\bibitem[Haarnoja et~al\mbox{.}(2018)]%
        {haarnoja2018soft}
\bibfield{author}{\bibinfo{person}{Tuomas Haarnoja}, \bibinfo{person}{Aurick Zhou}, \bibinfo{person}{Pieter Abbeel}, {and} \bibinfo{person}{Sergey Levine}.} \bibinfo{year}{2018}\natexlab{}.
\newblock \showarticletitle{Soft actor-critic: Off-policy maximum entropy deep reinforcement learning with a stochastic actor}. In \bibinfo{booktitle}{\emph{International conference on machine learning}}. PMLR, \bibinfo{pages}{1861--1870}.
\newblock


\bibitem[Hong et~al\mbox{.}(2020)]%
        {hong2020nonintrusive}
\bibfield{author}{\bibinfo{person}{Daocheng Hong}, \bibinfo{person}{Yang Li}, {and} \bibinfo{person}{Qiwen Dong}.} \bibinfo{year}{2020}\natexlab{}.
\newblock \showarticletitle{Nonintrusive-Sensing and Reinforcement-Learning Based Adaptive Personalized Music Recommendation}. In \bibinfo{booktitle}{\emph{Proceedings of the 43rd International ACM SIGIR Conference on Research and Development in Information Retrieval}}. \bibinfo{pages}{1721--1724}.
\newblock


\bibitem[Huang et~al\mbox{.}(2022)]%
        {huang2022action}
\bibfield{author}{\bibinfo{person}{Biwei Huang}, \bibinfo{person}{Chaochao Lu}, \bibinfo{person}{Liu Leqi}, \bibinfo{person}{Jos{\'e}~Miguel Hern{\'a}ndez-Lobato}, \bibinfo{person}{Clark Glymour}, \bibinfo{person}{Bernhard Sch{\"o}lkopf}, {and} \bibinfo{person}{Kun Zhang}.} \bibinfo{year}{2022}\natexlab{}.
\newblock \showarticletitle{Action-sufficient state representation learning for control with structural constraints}. In \bibinfo{booktitle}{\emph{International Conference on Machine Learning}}. PMLR, \bibinfo{pages}{9260--9279}.
\newblock


\bibitem[Kostrikov et~al\mbox{.}(2019)]%
        {DBLP:conf/iclr/KostrikovADLT19}
\bibfield{author}{\bibinfo{person}{Ilya Kostrikov}, \bibinfo{person}{Kumar~Krishna Agrawal}, \bibinfo{person}{Debidatta Dwibedi}, \bibinfo{person}{Sergey Levine}, {and} \bibinfo{person}{Jonathan Tompson}.} \bibinfo{year}{2019}\natexlab{}.
\newblock \showarticletitle{Discriminator-Actor-Critic: Addressing Sample Inefficiency and Reward Bias in Adversarial Imitation Learning}. In \bibinfo{booktitle}{\emph{7th International Conference on Learning Representations, {ICLR} 2019, New Orleans, LA, USA, May 6-9, 2019}}. \bibinfo{publisher}{OpenReview.net}.
\newblock
\urldef\tempurl%
\url{https://openreview.net/forum?id=Hk4fpoA5Km}
\showURL{%
\tempurl}


\bibitem[Lesort et~al\mbox{.}(2018)]%
        {lesort2018state}
\bibfield{author}{\bibinfo{person}{Timoth{\'e}e Lesort}, \bibinfo{person}{Natalia D{\'\i}az-Rodr{\'\i}guez}, \bibinfo{person}{Jean-Franois Goudou}, {and} \bibinfo{person}{David Filliat}.} \bibinfo{year}{2018}\natexlab{}.
\newblock \showarticletitle{State representation learning for control: An overview}.
\newblock \bibinfo{journal}{\emph{Neural Networks}}  \bibinfo{volume}{108} (\bibinfo{year}{2018}), \bibinfo{pages}{379--392}.
\newblock


\bibitem[Lillicrap et~al\mbox{.}(2015)]%
        {lillicrap2015continuous}
\bibfield{author}{\bibinfo{person}{Timothy~P Lillicrap}, \bibinfo{person}{Jonathan~J Hunt}, \bibinfo{person}{Alexander Pritzel}, \bibinfo{person}{Nicolas Heess}, \bibinfo{person}{Tom Erez}, \bibinfo{person}{Yuval Tassa}, \bibinfo{person}{David Silver}, {and} \bibinfo{person}{Daan Wierstra}.} \bibinfo{year}{2015}\natexlab{}.
\newblock \showarticletitle{Continuous control with deep reinforcement learning}.
\newblock \bibinfo{journal}{\emph{arXiv preprint arXiv:1509.02971}} (\bibinfo{year}{2015}).
\newblock


\bibitem[Liu et~al\mbox{.}(2020a)]%
        {liu2020general}
\bibfield{author}{\bibinfo{person}{Dugang Liu}, \bibinfo{person}{Pengxiang Cheng}, \bibinfo{person}{Zhenhua Dong}, \bibinfo{person}{Xiuqiang He}, \bibinfo{person}{Weike Pan}, {and} \bibinfo{person}{Zhong Ming}.} \bibinfo{year}{2020}\natexlab{a}.
\newblock \showarticletitle{A general knowledge distillation framework for counterfactual recommendation via uniform data}. In \bibinfo{booktitle}{\emph{Proceedings of the 43rd International ACM SIGIR Conference on Research and Development in Information Retrieval}}. \bibinfo{pages}{831--840}.
\newblock


\bibitem[Liu et~al\mbox{.}(2020b)]%
        {LIU2020106170}
\bibfield{author}{\bibinfo{person}{Feng Liu}, \bibinfo{person}{Ruiming Tang}, \bibinfo{person}{Xutao Li}, \bibinfo{person}{Weinan Zhang}, \bibinfo{person}{Yunming Ye}, \bibinfo{person}{Haokun Chen}, \bibinfo{person}{Huifeng Guo}, \bibinfo{person}{Yuzhou Zhang}, {and} \bibinfo{person}{Xiuqiang He}.} \bibinfo{year}{2020}\natexlab{b}.
\newblock \showarticletitle{State representation modeling for deep reinforcement learning based recommendation}.
\newblock \bibinfo{journal}{\emph{Knowledge-Based Systems}}  \bibinfo{volume}{205} (\bibinfo{year}{2020}), \bibinfo{pages}{106170}.
\newblock
\showISSN{0950-7051}
\urldef\tempurl%
\url{https://doi.org/10.1016/j.knosys.2020.106170}
\showDOI{\tempurl}


\bibitem[Madumal et~al\mbox{.}(2020)]%
        {madumal2020explainable}
\bibfield{author}{\bibinfo{person}{Prashan Madumal}, \bibinfo{person}{Tim Miller}, \bibinfo{person}{Liz Sonenberg}, {and} \bibinfo{person}{Frank Vetere}.} \bibinfo{year}{2020}\natexlab{}.
\newblock \showarticletitle{Explainable reinforcement learning through a causal lens}. In \bibinfo{booktitle}{\emph{Proceedings of the AAAI conference on artificial intelligence}}, Vol.~\bibinfo{volume}{34}. \bibinfo{pages}{2493--2500}.
\newblock


\bibitem[Mahmood and Ricci(2007)]%
        {mahmood2007learning}
\bibfield{author}{\bibinfo{person}{Tariq Mahmood} {and} \bibinfo{person}{Francesco Ricci}.} \bibinfo{year}{2007}\natexlab{}.
\newblock \showarticletitle{Learning and adaptivity in interactive recommender systems}. In \bibinfo{booktitle}{\emph{Proceedings of the ninth international conference on Electronic commerce}}. \bibinfo{pages}{75--84}.
\newblock


\bibitem[Mastakouri et~al\mbox{.}(2021)]%
        {mastakouri2021necessary}
\bibfield{author}{\bibinfo{person}{Atalanti~A Mastakouri}, \bibinfo{person}{Bernhard Sch{\"o}lkopf}, {and} \bibinfo{person}{Dominik Janzing}.} \bibinfo{year}{2021}\natexlab{}.
\newblock \showarticletitle{Necessary and sufficient conditions for causal feature selection in time series with latent common causes}. In \bibinfo{booktitle}{\emph{International Conference on Machine Learning}}. PMLR, \bibinfo{pages}{7502--7511}.
\newblock


\bibitem[Pearl(2009)]%
        {pearl2009causality}
\bibfield{author}{\bibinfo{person}{Judea Pearl}.} \bibinfo{year}{2009}\natexlab{}.
\newblock \bibinfo{booktitle}{\emph{Causality}}.
\newblock \bibinfo{publisher}{Cambridge university press}.
\newblock


\bibitem[Peters et~al\mbox{.}(2017)]%
        {10.5555/3202377}
\bibfield{author}{\bibinfo{person}{Jonas Peters}, \bibinfo{person}{Dominik Janzing}, {and} \bibinfo{person}{Bernhard Schlkopf}.} \bibinfo{year}{2017}\natexlab{}.
\newblock \bibinfo{booktitle}{\emph{Elements of Causal Inference: Foundations and Learning Algorithms}}.
\newblock \bibinfo{publisher}{The MIT Press}.
\newblock
\showISBNx{0262037319}


\bibitem[Shi et~al\mbox{.}(2019)]%
        {shi2019virtual}
\bibfield{author}{\bibinfo{person}{Jing-Cheng Shi}, \bibinfo{person}{Yang Yu}, \bibinfo{person}{Qing Da}, \bibinfo{person}{Shi-Yong Chen}, {and} \bibinfo{person}{An-Xiang Zeng}.} \bibinfo{year}{2019}\natexlab{}.
\newblock \showarticletitle{Virtual-taobao: Virtualizing real-world online retail environment for reinforcement learning}. In \bibinfo{booktitle}{\emph{Proceedings of the AAAI Conference on Artificial Intelligence}}, Vol.~\bibinfo{volume}{33}. \bibinfo{pages}{4902--4909}.
\newblock


\bibitem[Spirtes et~al\mbox{.}(2000)]%
        {spirtes2000causation}
\bibfield{author}{\bibinfo{person}{Peter Spirtes}, \bibinfo{person}{Clark~N Glymour}, {and} \bibinfo{person}{Richard Scheines}.} \bibinfo{year}{2000}\natexlab{}.
\newblock \bibinfo{booktitle}{\emph{Causation, prediction, and search}}.
\newblock \bibinfo{publisher}{MIT press}.
\newblock


\bibitem[Tan et~al\mbox{.}(2021)]%
        {tan2021counterfactual}
\bibfield{author}{\bibinfo{person}{Juntao Tan}, \bibinfo{person}{Shuyuan Xu}, \bibinfo{person}{Yingqiang Ge}, \bibinfo{person}{Yunqi Li}, \bibinfo{person}{Xu Chen}, {and} \bibinfo{person}{Yongfeng Zhang}.} \bibinfo{year}{2021}\natexlab{}.
\newblock \showarticletitle{Counterfactual explainable recommendation}. In \bibinfo{booktitle}{\emph{Proceedings of the 30th ACM International Conference on Information \& Knowledge Management}}. \bibinfo{pages}{1784--1793}.
\newblock


\bibitem[Wang et~al\mbox{.}(2023a)]%
        {wang2023causal}
\bibfield{author}{\bibinfo{person}{Siyu Wang}, \bibinfo{person}{Xiaocong Chen}, \bibinfo{person}{Dietmar Jannach}, {and} \bibinfo{person}{Lina Yao}.} \bibinfo{year}{2023}\natexlab{a}.
\newblock \showarticletitle{Causal decision transformer for recommender systems via offline reinforcement learning}. In \bibinfo{booktitle}{\emph{Proceedings of the 46th International ACM SIGIR Conference on Research and Development in Information Retrieval}}. \bibinfo{pages}{1599--1608}.
\newblock


\bibitem[Wang et~al\mbox{.}(2023b)]%
        {wang2023plug}
\bibfield{author}{\bibinfo{person}{Siyu Wang}, \bibinfo{person}{Xiaocong Chen}, \bibinfo{person}{Julian McAuley}, \bibinfo{person}{Sally Cripps}, {and} \bibinfo{person}{Lina Yao}.} \bibinfo{year}{2023}\natexlab{b}.
\newblock \showarticletitle{Plug-and-Play Model-Agnostic Counterfactual Policy Synthesis for Deep Reinforcement Learning-Based Recommendation}.
\newblock \bibinfo{journal}{\emph{IEEE Transactions on Neural Networks and Learning Systems}} (\bibinfo{year}{2023}).
\newblock


\bibitem[Wang et~al\mbox{.}(2021)]%
        {wang2021clicks}
\bibfield{author}{\bibinfo{person}{Wenjie Wang}, \bibinfo{person}{Fuli Feng}, \bibinfo{person}{Xiangnan He}, \bibinfo{person}{Hanwang Zhang}, {and} \bibinfo{person}{Tat-Seng Chua}.} \bibinfo{year}{2021}\natexlab{}.
\newblock \showarticletitle{Clicks can be cheating: Counterfactual recommendation for mitigating clickbait issue}. In \bibinfo{booktitle}{\emph{Proceedings of the 44th International ACM SIGIR Conference on Research and Development in Information Retrieval}}. \bibinfo{pages}{1288--1297}.
\newblock


\bibitem[Wang et~al\mbox{.}(2022a)]%
        {wang2022causal}
\bibfield{author}{\bibinfo{person}{Wenjie Wang}, \bibinfo{person}{Xinyu Lin}, \bibinfo{person}{Fuli Feng}, \bibinfo{person}{Xiangnan He}, \bibinfo{person}{Min Lin}, {and} \bibinfo{person}{Tat-Seng Chua}.} \bibinfo{year}{2022}\natexlab{a}.
\newblock \showarticletitle{Causal Representation Learning for Out-of-Distribution Recommendation}. In \bibinfo{booktitle}{\emph{Proceedings of the ACM Web Conference 2022}}. \bibinfo{pages}{3562--3571}.
\newblock


\bibitem[Wang(2020)]%
        {wang2020hybrid}
\bibfield{author}{\bibinfo{person}{Yu Wang}.} \bibinfo{year}{2020}\natexlab{}.
\newblock \showarticletitle{A hybrid recommendation for music based on reinforcement learning}. In \bibinfo{booktitle}{\emph{Advances in Knowledge Discovery and Data Mining: 24th Pacific-Asia Conference, PAKDD 2020, Singapore, May 11--14, 2020, Proceedings, Part I 24}}. Springer, \bibinfo{pages}{91--103}.
\newblock


\bibitem[Wang et~al\mbox{.}(2022b)]%
        {pmlr-v162-wang22ae}
\bibfield{author}{\bibinfo{person}{Zizhao Wang}, \bibinfo{person}{Xuesu Xiao}, \bibinfo{person}{Zifan Xu}, \bibinfo{person}{Yuke Zhu}, {and} \bibinfo{person}{Peter Stone}.} \bibinfo{year}{2022}\natexlab{b}.
\newblock \showarticletitle{Causal Dynamics Learning for Task-Independent State Abstraction}. In \bibinfo{booktitle}{\emph{Proceedings of the 39th International Conference on Machine Learning}} \emph{(\bibinfo{series}{Proceedings of Machine Learning Research}, Vol.~\bibinfo{volume}{162})}, \bibfield{editor}{\bibinfo{person}{Kamalika Chaudhuri}, \bibinfo{person}{Stefanie Jegelka}, \bibinfo{person}{Le~Song}, \bibinfo{person}{Csaba Szepesvari}, \bibinfo{person}{Gang Niu}, {and} \bibinfo{person}{Sivan Sabato}} (Eds.). \bibinfo{publisher}{PMLR}, \bibinfo{pages}{23151--23180}.
\newblock
\urldef\tempurl%
\url{https://proceedings.mlr.press/v162/wang22ae.html}
\showURL{%
\tempurl}


\bibitem[Xian et~al\mbox{.}(2019)]%
        {xian2019reinforcement}
\bibfield{author}{\bibinfo{person}{Yikun Xian}, \bibinfo{person}{Zuohui Fu}, \bibinfo{person}{S Muthukrishnan}, \bibinfo{person}{Gerard de Melo}, {and} \bibinfo{person}{Yongfeng Zhang}.} \bibinfo{year}{2019}\natexlab{}.
\newblock \showarticletitle{Reinforcement Knowledge Graph Reasoning for Explainable Recommendation}. In \bibinfo{booktitle}{\emph{Proceedings of the 42nd International ACM SIGIR Conference on Research and Development in Information Retrieval}}. ACM, \bibinfo{pages}{285--294}.
\newblock


\bibitem[Zhang et~al\mbox{.}(2021b)]%
        {zhang2021learning}
\bibfield{author}{\bibinfo{person}{Amy Zhang}, \bibinfo{person}{Rowan~Thomas McAllister}, \bibinfo{person}{Roberto Calandra}, \bibinfo{person}{Yarin Gal}, {and} \bibinfo{person}{Sergey Levine}.} \bibinfo{year}{2021}\natexlab{b}.
\newblock \showarticletitle{Learning Invariant Representations for Reinforcement Learning without Reconstruction}. In \bibinfo{booktitle}{\emph{International Conference on Learning Representations}}.
\newblock
\urldef\tempurl%
\url{https://openreview.net/forum?id=-2FCwDKRREu}
\showURL{%
\tempurl}


\bibitem[Zhang et~al\mbox{.}(2021c)]%
        {zhang2021causerec}
\bibfield{author}{\bibinfo{person}{Shengyu Zhang}, \bibinfo{person}{Dong Yao}, \bibinfo{person}{Zhou Zhao}, \bibinfo{person}{Tat-Seng Chua}, {and} \bibinfo{person}{Fei Wu}.} \bibinfo{year}{2021}\natexlab{c}.
\newblock \showarticletitle{Causerec: Counterfactual user sequence synthesis for sequential recommendation}. In \bibinfo{booktitle}{\emph{Proceedings of the 44th International ACM SIGIR Conference on Research and Development in Information Retrieval}}. \bibinfo{pages}{367--377}.
\newblock


\bibitem[Zhang et~al\mbox{.}(2021a)]%
        {zhang2021causal}
\bibfield{author}{\bibinfo{person}{Yang Zhang}, \bibinfo{person}{Fuli Feng}, \bibinfo{person}{Xiangnan He}, \bibinfo{person}{Tianxin Wei}, \bibinfo{person}{Chonggang Song}, \bibinfo{person}{Guohui Ling}, {and} \bibinfo{person}{Yongdong Zhang}.} \bibinfo{year}{2021}\natexlab{a}.
\newblock \showarticletitle{Causal intervention for leveraging popularity bias in recommendation}. In \bibinfo{booktitle}{\emph{Proceedings of the 44th International ACM SIGIR Conference on Research and Development in Information Retrieval}}. \bibinfo{pages}{11--20}.
\newblock


\bibitem[Zhang et~al\mbox{.}(2017)]%
        {zhang2017dynamic}
\bibfield{author}{\bibinfo{person}{Yang Zhang}, \bibinfo{person}{Chenwei Zhang}, {and} \bibinfo{person}{Xiaozhong Liu}.} \bibinfo{year}{2017}\natexlab{}.
\newblock \showarticletitle{Dynamic scholarly collaborator recommendation via competitive multi-agent reinforcement learning}. In \bibinfo{booktitle}{\emph{Proceedings of the eleventh ACM conference on recommender systems}}. \bibinfo{pages}{331--335}.
\newblock


\bibitem[Zhao et~al\mbox{.}(2018)]%
        {zhao2018recommendations}
\bibfield{author}{\bibinfo{person}{Xiangyu Zhao}, \bibinfo{person}{Liang Zhang}, \bibinfo{person}{Zhuoye Ding}, \bibinfo{person}{Long Xia}, \bibinfo{person}{Jiliang Tang}, {and} \bibinfo{person}{Dawei Yin}.} \bibinfo{year}{2018}\natexlab{}.
\newblock \showarticletitle{Recommendations with negative feedback via pairwise deep reinforcement learning}. In \bibinfo{booktitle}{\emph{Proceedings of the 24th ACM SIGKDD International Conference on Knowledge Discovery \& Data Mining}}. \bibinfo{pages}{1040--1048}.
\newblock


\end{thebibliography}

\end{document}